\documentclass{article}
\usepackage{arxiv}

\usepackage{amsthm}
\usepackage{hyperref}
\usepackage{url}
\usepackage{hyperref}
\usepackage{url}
\usepackage[utf8]{inputenc} 
\usepackage[T1]{fontenc}    
\usepackage{booktabs}       
\usepackage{amsfonts}       
\usepackage{nicefrac}       
\usepackage{microtype}      
\usepackage{mathtools}
\usepackage{amssymb}
\usepackage[shortlabels]{enumitem}
\usepackage{xcolor,colortbl}
\usepackage{comment}
\usepackage{subcaption}
\usepackage{graphicx}
\usepackage{makecell}
\usepackage{multirow}
\usepackage{float}
\usepackage{booktabs,caption}
\usepackage[flushleft]{threeparttable}
\usepackage{ulem}
\usepackage{algorithm}
\usepackage{algpseudocode}
\usepackage[T1]{fontenc}
\usepackage{lmodern}

\newtheorem{theorem}{Theorem}

\newtheorem{remark}[theorem]{Remark}
\newtheorem{lemma}[theorem]{Lemma}

\newtheorem{assumption}[theorem]{Assumption}
\newtheorem{proposition}[theorem]{Proposition}

\numberwithin{theorem}{section}

\title{Statistical Guarantees for Distributionally Robust Optimization with Optimal Transport and OT-Regularized Divergences}

\author{Jeremiah Birrell\\
  Department of Mathematics\\
  Texas State University\\
  San Marcos, TX,  USA \\
  \texttt{jbirrell@txstate.edu} \\
   \And
Xiaoxi Shen\\
  Department of Mathematics\\
  Texas State University\\
  San Marcos, TX,  USA \\
  \texttt{rcd67@txstate.edu} \\
}
\date{\today}

\begin{document}

\maketitle

\begin{abstract}
We study finite-sample statistical performance guarantees for distributionally robust optimization (DRO) with optimal transport (OT) and OT-regularized divergence model neighborhoods.  Specifically, we derive concentration inequalities for supervised learning via DRO-based adversarial training, as commonly employed to enhance the adversarial robustness of machine learning models.  Our results apply to a wide range of OT cost functions, beyond the $p$-Wasserstein case studied by previous authors.  In particular, our results are the first to: 1) cover soft-constraint norm-ball OT cost functions; soft-constraint costs have been shown empirically to enhance  robustness when used in adversarial training, 2) apply to the combination of adversarial sample generation and adversarial reweighting that is induced by using OT-regularized  $f$-divergence model neighborhoods; the added reweighting mechanism  has also been shown empirically to further improve performance.  In addition, even in the $p$-Wasserstein case, our bounds exhibit better behavior as a function of the DRO neighborhood size than previous results when applied to the adversarial setting.
\end{abstract}


\section{Introduction}

Distributionally robust optimization (DRO) is a technique for regularizing a stochastic optimization problem, $\inf_{\theta}E_{P}[\mathcal{L}_\theta]$, by replacing the $P$-expectation with the worst-case expected value over some neighborhood of `nearby' models, $\mathcal{U}(P)$, also known as the ambiguity set or uncertainty region in some literature. This results in the minimax problem
\begin{align}\label{eq:DRO_gen}
    \inf_\theta \sup_{Q\in\mathcal{U}(P)}E_Q[\mathcal{L}_\theta]\,.
\end{align}
A variety of model neighborhoods have been studied and employed in prior work, including  maximum mean discrepancy (MMD) \cite{NEURIPS2019_1770ae9e},   $f$-divergence neighborhoods \cite{doi:10.1287/opre.1100.0821,Javid2012,Hu2013,10.2307/23359484,doi:10.1287/opre.2018.1786},  (conditional) moment constraints \cite{doi:10.1287/opre.1090.0795,doi:10.1287/opre.1090.0741,doi:10.1287/opre.2014.1314,2023arXiv230805414B}, smoothed $f$-divergences \cite{liu2023smoothed},  Sinkhorn divergence \cite{wang2025sinkhorn},  Wasserstein neighborhoods \cite{EsfahaniKuhn2018,shafieezadeh2019regularization,wu2022generalization,li2022general,doi:10.1287/moor.2022.1275},  and  general optimal-transport (OT) neighborhoods \cite{doi:10.1287/moor.2018.0936,azizian2023regularization}.    

DRO has been used effectively in a number of applications; see, e.g.,  \cite{kuhn2025distributionally} for an overview. Our  focus is on adversarial robustness in machine learning; the vulnerability  to adversarial samples is a well known weakness  of machine learning models (especially deep learning) \cite{papernot2016limitations,goodfellow2014explaining}. Adversarial samples are  inputs that are intentionally modified by an adversary to mislead the model, e.g., causing a harmful misclassification of the input. Adversarial training is a popular class of techniques for mitigating this issue; in adversarial training,  adversarial samples are constructed in a way that mimics the efforts of an attacker and are then employed during training.  Having such a simulated attacker compete with the model during training results in a more robust trained model   \cite{papernot2017mask,madry2018towards,hu2018does,wang2019improving_mart,zhang2019theoreticallytradeOff,zhang2020geometry,dong2020adversarial,regniez2021distributional,bui_UDR_2022unified,dong2023towards,birrell2025optimal}. Many commonly used adversarial training methods can be formulated as DRO problems \eqref{eq:DRO_gen}.

Here we focus on adversarial training based on  DRO with OT and OT-regularized divergence neighborhoods.  OT-regularized divergences were recently introduced in \cite{birrell2025optimal} for the purpose of enhancing  adversarial robustness in deep  learning.  Specifically, we consider OT-regularized $f$-divergences, defined via an infimal convolution of a OT cost, $C$, and an $f$-divergence, $D_f$, as follows:
\begin{align}\label{eq:D_f_c_def}
D_f^c(\nu\|\mu)\coloneqq \inf_{\eta\in\mathcal{P}(\mathcal{Z})}\{D_f(\eta\|\mu)+C(\eta,\nu)\}\,,
\end{align}   
where  $\mathcal{P}(\mathcal{Z})$ denotes the space of  probability measures on  $\mathcal{Z}$. In \cite{birrell2025optimal} it was shown that DRO  with OT-regularized $f$-divergence neighborhoods corresponds to adversarial sample generation (due to the OT cost) along with adversarial sample reweighting (due to the $f$-divergence); combining these two mechanisms   enhances previous DRO-based approaches to adversarial training,  such as  PGD   \cite{madry2018towards}, TRADES  \cite{zhang2019theoreticallytradeOff},  MART \cite{wang2019improving_mart}, and UDR  \cite{bui_UDR_2022unified}.

In this work we derive finite-sample statistical guarantees, in the form of concentration inequalities, for DRO with OT and OT-regularized $f$-divergences, as applied to supervised learning.    Our main contributions are the following:
\begin{enumerate}
    \item We derive concentration inequalities for OT-DRO that apply to a much more general class of OT cost functions than considered in previous works, in both the classification and regression settings; see Theorems \ref{thm:conc_ineq_min_value} and \ref{thm:OT_DRO_emp_optimizer}. The increased generality facilitates applications to adversarial training methods which were  not covered by prior approaches.  In addition, even in the $p$-Wasserstein case studied in previous works, our results have better dependence on the neighborhood size than  other methods that treat the robust training setting; see Remark \ref{remark:r_dependence}.
    \item  We provide the first  analysis of  statistical guarantees for OT-regularized $f$-divergence DRO, as used to enhance the adversarial robustness of classification models; see Theorems \ref{thm:OT_reg_DRO_emp_objective_bound} and
\ref{thm:OT_reg_DRO_conc}.   This recent class of DRO methods  combines sample reweighting with adversarial sample generation and has been shown empirically to improve adversarial training performance.  
 \end{enumerate}

\subsection{Comparison with Related Works}
Statistical guarantees for DRO with  model neighborhoods of various types have been studied previously by a number of authors.  First we note the methods based on   Wasserstein-metric neighborhoods \cite{lee2018minimax,shafieezadeh2019regularization,an2021generalization,blanchet2022confidence,gao2023finite,azizian2023exact,gao2024wasserstein,aolaritei2026wasserstein}; see also the tutorial article \cite{blanchet2021statistical}.  Of these, \cite{lee2018minimax} focus on a fixed neighborhood size $r$, which is the appropriate setting for adversarial training,  \cite{shafieezadeh2019regularization,blanchet2022confidence,gao2023finite,aolaritei2026wasserstein} focus on   neighborhood-size that shrinks to $0$ as the number of samples $n\to\infty$, as is appropriate for use as a regularization term to protect against overfitting, and \cite{an2021generalization,azizian2023exact,gao2024wasserstein} cover both cases.  However, to the authors' knowledge, there are no existing works showing convergence of the empirical OT-DRO problem to population problem which cover more general OT-cost functions, beyond the $p$-Wasserstein family. In this work we consider a  class of OT cost functions that simultaneously generalize $p$-Wasserstein and norm-ball constraint costs and focus on bounds that behave well for a fixed, small neighborhood size; both factors  facilitate applications to adversarial robustness.  

Another stream of work focuses on  information-theoretic model neighborhoods \cite{ben2013robust,lam2017empirical,bertsimas2018data,duchi2019variance,an2021generalization,duchi2021statistics,duchi2021learning}.  Of these, the closest to the present work is  \cite{duchi2021learning}, which obtained finite-sample statistical guarantees   for DRO with Cressie-Read divergences, a family of $f$-divergences which are closely related to the $\alpha$-divergences \eqref{eq:f_alpha_def}, as a method for addressing distributional shifts.  However, information-theoretic neighborhoods cannot account for a change in the support of the distribution, as is required for applications to adversarial robustness.   To unify the OT and information theoretic approaches,  our statistical guarantees apply to the OT-regularized $f$-divergences \eqref{eq:D_f_c_def}, developed in \cite{birrell2025optimal} for application to adversarial robustness; these divergences   interpolate between the OT and information-theoretic approaches, thereby allowing for adversarial sample generation along with sample reweighting.  Our simultaneous study of more general OT-cost and general $f$-divergences requires several innovative techniques. In particular, the Cressie-Read divergence case allows for an analytical simplification that is not possible for the more general class of  $f$-divergences    covered by our results.

Finally, we mention several other  recent methods that combine OT and information-theoretic ingredients in DRO.  The approach  of \cite{azizian2023exact} studies Wasserstein DRO with a KL-divergence penalty between  the transport plan and a Gaussian mixture centered on the samples; this is a very different approach to the two-stage transport-reweighting mechanism inherent in our OT-regularized divergences, both conceptually and computationally. Closer to our approach are those of \cite{2023arXiv230805414B,liu2023smoothed}. While distinct in general, the approach of \cite{2023arXiv230805414B} overlaps with the OT-regularized divergence framework \cite{birrell2025optimal} studied here in the $f$-divergence case, but \cite{2023arXiv230805414B} does not consider the statistical properties of the method.  The frameworks of  \cite{liu2023smoothed} and \cite{birrell2025optimal} are distinct when applied to $f$-divergences, and the former derives statistical guarantees that apply to smoothing via $1$-Wasserstein and  L{\'e}vy-Prokhorov metrics (see their Section 5.2); specifically,  they show that their empirical DRO problem upper bounds the population DRO problem with high-probability.  This   conclusion is weaker than what is produced by the results in this paper, which show convergence of the empirical problem to the population problem as $n\to\infty$; the latter type of result is more meaningful in the adversarial training setting. Moreover, we obtain results for more general OT cost functions.  The  OT-regularized divergence approach studied here has already been shown to be a practical computational tool for enhancing the  adversarial robustness of machine learning models, further motivating the study of its statistical properties in settings adapted to that application.

\section{Background and Main Results}\label{sec:result_summary}
In this section we provide the necessary background on DRO with OT and OT-regularized $f$-divergences and summarize our main results.  Details regarding the required technical assumptions, along with the proofs, can be found in the subsequent sections.  
\subsection{Background}
We let $\mathcal{Z}$ denote a Polish space (i.e., a complete separable metric space)    and let  $\mathcal{P}(\mathcal{Z})$   denote the space of Borel probability measures on $\mathcal{Z}$.    We will use the term   cost-function to refer to   a lower semicontinuous (LSC) function $c:\mathcal{Z}\times\mathcal{Z}\to[0,\infty]$.  The associated   optimal-transport (OT) cost is defined by $C:\mathcal{P}(\mathcal{Z})\times\mathcal{P}(\mathcal{Z})\to[0,\infty]$, \begin{align}
    C(\mu,\nu)\coloneqq \inf_{\substack{\pi\in\mathcal{P}(\mathcal{Z}\times\mathcal{Z}):\\ \pi_1=\mu,\pi_2=\nu}}\int cd\pi\,,
    \end{align}
    where  $\pi_1,\pi_2$ denote the marginal distributions; see, e.g., \cite{villani2008optimal} for background on optimal transport.  The following result, adapted from \cite{doi:10.1287/moor.2018.0936,doi:10.1287/moor.2022.1275,zhang2025short}, provides an important dual formulation of the OT-DRO problem.
\begin{proposition}\label{prop:OT_DRO}
   Let   $c$ be a cost function that satisfies $c(z,z)=0$ for all $z\in\mathcal{Z}$,  $\mathcal{L}:\mathcal{Z}\to\mathbb{R}$ be measurable and bounded below, and  $P\in\mathcal{P}(\mathcal{Z})$.  Then for all  $r>0$ we have
\begin{align}\label{eq:OT_DRO}
&\sup_{Q:C(P,Q)\leq r}E_Q[\mathcal{L}]
=\inf_{\lambda>0}\{\lambda r+ E_{{P}}[\mathcal{L}^c_{\lambda}]\}\,,
\end{align}  
where
\begin{align}\label{eq:L_c_def}
&  \mathcal{L}^c_{\lambda}(z)\coloneqq  \sup_{\tilde z\in\mathcal{Z}}\{ \mathcal{L} (\tilde z) -  \lambda c(z,\tilde z)\}
\end{align}
and we employ the convention $\infty-\infty\coloneqq-\infty$.
\end{proposition}
\begin{remark}
Note that $\mathcal{L}^c_\lambda$ is universally measurable; we do not distinguish between $P$ and its completion in our notation. $\mathcal{L}^c_{\lambda}$   is known as the $c$-transform in the optimal transport literature; see Definition  5.2 in \cite{villani2008optimal}.
\end{remark}
In practice, the supreumum over $\tilde{z}$ in \eqref{eq:L_c_def} represents the generation of an adversarial sample, $\tilde{z}$, paired with the original sample, $z$, i.e., the action of the simulated attacker.

    The OT-regularized divergences, \eqref{eq:D_f_c_def}, mix an OT cost with an $f$-divergence, with the latter defined as follows. For $a,b\in[-\infty,\infty]$ that satisfy $-\infty\leq a<1<b\leq\infty$ we define $\mathcal{F}_1(a,b)$ to be the set of convex functions $f:(a,b)\to\mathbb{R}$ with $f(1)=0$.  Given $f\in\mathcal{F}_1(a,b)$, the corresponding  $f$-divergence between $\nu,\mu\in\mathcal{P}(\mathcal{Z})$ is defined by
\begin{align}\label{eq:f_div_def}
D_f(\nu\|\mu)=\begin{cases} 
     E_P[f(d\nu/d\mu)], & \nu\ll \mu\\
      \infty, &\nu\not\ll \mu
   \end{cases}\,,
\end{align}
where the definition of $f$ in \eqref{eq:f_div_def} is extended to $[a,b]$ by continuity and is set to $\infty$ on $[a,b]^c$.  The key  result from \cite{birrell2025optimal} that we will require is the following dual formulation of the OT-regularized $f$-divergence DRO problem. 
\begin{proposition}\label{prop:OT_f_div_DRO}
Suppose we have the following:
\begin{enumerate}
\item A measurable function $\mathcal{L}:\mathcal{Z}\to(-\infty,\infty]$ that is  bounded below.
\item $P\in\mathcal{P}(\mathcal{Z})$.
\item $f\in\mathcal{F}_1(a,b)$, where $a\geq 0$.
\item  A  cost function, $c$, that satisfies $c(z,z)=0$  for all $z\in\mathcal{Z}$.
\end{enumerate}
 Then for $r>0$  we have
\begin{align}\label{eq:OT_f_div_DRO}
&\sup_{Q:D_f^c(Q\|P)\leq r}E_Q[\mathcal{L}]
=\inf_{\lambda>0,\rho\in\mathbb{R}}\{\lambda r+\rho+\lambda E_{{P}}[f^*((\mathcal{L}^c_{\lambda}-\rho)/\lambda)]\}\,,
\end{align}  
where $\mathcal{L}_\lambda^c$ was defined in \eqref{eq:L_c_def} and we employ the convention  $f^*(\infty)\coloneqq\infty$.  
\end{proposition}
 
Important examples of OT-regularized $f$-divergences  include  those constructed using the $\alpha$-divergences, defined in terms of
\begin{align}\label{eq:f_alpha_def}
f_\alpha(t)=      \frac{t^\alpha-1}{\alpha(\alpha-1)},\,\,\,\,\, \alpha>1\,,
\end{align}
and which has Legendre transform    
\begin{align}\label{eq:f_alpha_star}
f_\alpha^*(t)=\alpha^{-1}(\alpha-1)^{\alpha/(\alpha-1)}\max\{t,0\}^{\alpha/(\alpha-1)}+\frac{1}{\alpha(\alpha-1)},\,\,\,\,\,\alpha>1\,,
\end{align}
along with the  KL divergence case, defined using $f_{KL}(t)=t\log(t)$, for which one has the  simplification
\begin{align}\label{eq:OT_KL_DRO}
\sup_{Q: KL^c(Q\|P)\leq r}E_Q[\mathcal{L}]=\inf_{\lambda>0}\left\{ \lambda r+\lambda \log\left(E_P[\exp(\lambda^{-1}\mathcal{L}_{\lambda}^{c})]\right)\right\}\,.
\end{align}

Practical implementations of DRO for adversarial training utilize the dual formulations, i.e., the right-hand sides of \eqref{eq:OT_DRO},  \eqref{eq:OT_f_div_DRO}, and \eqref{eq:OT_KL_DRO}. Therefore, these will be our focus for the remainder of this paper. 

\subsection{Summary of the Main Results}
Next we summarize the main results in this paper; details regarding the required assumptions and proofs can be found in Sections \ref{sec:OT_DRO_proof}-\ref{sec:OT_Df_DRO_proof}.  Given our focus on supervised learning, samples $z\in \mathcal{Z}$ will have the form $z=(x,y)$, where  $x\in\mathbb{R}^d$ is the predictor and $y$ the response (i.e., label).  For OT-DRO we consider a general $y$ (i.e., either regression or classification), while for OT-regularized $f$-divergence DRO our approach is applicable to discrete $y$ (classification) only.
 
First, in Theorem \ref{thm:conc_ineq_min_value}, we prove that the optimal value of the $n$-sample empirical  OT-DRO problem  is close to the corresponding population value with high probability:
\begin{align}
&{P}^n\left(\pm\left(\inf_{\theta\in\Theta}\sup_{Q: C(P,Q)\leq r}E_Q[\mathcal{L}_\theta]-\inf_{\theta\in\Theta}\sup_{Q: C(P_{n},Q)\leq r}E_Q[\mathcal{L}_\theta]\right)\geq 2D_n+\epsilon\right)\\
\leq&\exp\left(-\frac{2\epsilon^2n}{\beta^2}\right)\,\,\,\,\,\text{ for all   $n\in\mathbb{Z}^+$,  $\epsilon>0$, $r>0$,}\notag
\end{align}
where $P_n$ denotes the empirical distribution corresponding to the    samples, $z_i\sim P$, $i=1,..,n$, $P^n$ denotes the product measure (i.e., the samples are i.i.d.), $\beta$ is an upper bound on $\mathcal{L}_\theta$, $\theta\in\Theta$, and $D_n$, defined in \eqref{eq:D_n_def}, is a measure of the   complexity of the function class $\{\mathcal{L}_\theta:\theta\in\Theta\}$; under appropriate assumptions, $D_n$ approaches $0$ as $n\to\infty$.   We build on this result in Theorem \ref{thm:OT_DRO_emp_optimizer} by showing that a solution to the empirical risk minimization (ERM) OT-DRO problem is also an approximate solution to the population DRO problem with high probability.  

Our Theorems \ref{thm:conc_ineq_min_value} and \ref{thm:OT_DRO_emp_optimizer}  apply to a wide range of OT cost functions, beyond the $p$-Wasserstein case  that has been studied in  previous works \cite{lee2018minimax,shafieezadeh2019regularization,an2021generalization,blanchet2022confidence,gao2023finite,azizian2023exact,gao2024wasserstein,aolaritei2026wasserstein}.  In particular, the increased generality of our results facilitates  applications to  adversarial   training  methods  from \cite{madry2018towards,zhang2019theoreticallytradeOff,wang2019improving_mart,bui_UDR_2022unified}.  The aforementioned adversarial training methods are either based on the PGD cost function \begin{align}\label{eq:PGD_cost}
    c_\delta((x,y),(\tilde{x},\tilde{y}))\coloneqq\infty1_{\|x-\tilde{x}\|_{\mathcal{X}}>\delta}+\infty1_{y\neq \tilde{y}}\,,\,\,\,\delta\geq 0\,,
\end{align}   or on soft-constraint relaxations of it. In particular, \cite{bui_UDR_2022unified} proposed a general class of soft-constraint relaxations of  \eqref{eq:PGD_cost}; in practice, the soft constraint allows the adversarial sample, $\tilde{x}$, to leave the $\delta$-ball centered at $x$, but only if there is an especially effective adversarial sample that is nearby.  In \cite{bui_UDR_2022unified}  it was observed empirically that using such relaxations for adversarial training leads to improved adversarial robustness. In this paper we will analyze relaxations of \eqref{eq:PGD_cost} that have the form 
\begin{align}\label{eq:c_psi_delta_def}
c_{\psi,\delta}((x,y),(\tilde{x},\tilde{y}))\coloneqq\phi_{\psi,\delta}(\|\tilde{x}-x\|_{\mathcal{X}})+\infty1_{y\neq \tilde{y}}\,,\,\,\,\phi_{\psi,\delta}(t)\coloneqq\psi(t-\delta)1_{t\geq\delta}\,,
\end{align}
for $\delta>0$ and appropriate choices of $\psi$. The function $\psi$ determines how much the adversarial sample is penalized when it leaves the $\delta$-ball; see Section \ref{sec:examples_OT_DRO} for specific  examples covered by our results.  We note that  our framework can be applied to methods  which require maintaining an unchanged copy of the original $x$ in addition to the adversarial $\tilde{x}$, e.g.,  TRADES  \cite{zhang2019theoreticallytradeOff} and MART \cite{wang2019improving_mart}, by considering the unchanged copy of $x$ to be part of the response, $y$.   We emphasize that cost functions of the form \eqref{eq:c_psi_delta_def} are not covered by previous statistical performance guarantees on OT-DRO; addressing this limitation is one of the main motivations for our analysis.

In \cite{birrell2025optimal} it was further demonstrated that combining  the OT costs used by \cite{bui_UDR_2022unified} with sample reweighting, in the form OT-regularized $f$-divergence DRO, further improves performance; the reweighting mechanism causes the training to focus more on the most troublesome adversarial samples, leading to increased adversarial robustness.  This motivates our study of the corresponding statistical guarantees; in Theorem \ref{thm:OT_reg_DRO_emp_objective_bound} we obtain the concentration inequality
 \begin{align}\label{eq:OT_rec_conc_intro}
&P^n\left(\pm\left(\inf_{\theta\in\Theta}\sup_{Q:D_f^c(Q\|P)\leq r}E_Q[\mathcal{L}_\theta]-\inf_{\theta\in\Theta}\sup_{Q:D_f^c(Q\|P_n)\leq r}E_Q[\mathcal{L}_\theta]\right)\geq \max\{R_n,\widetilde{R}_n\}+\epsilon\right)\\
\leq&\exp\left(-\frac{2n\epsilon^2}{ \beta^2}\right)+
\exp\left(-\frac{2n\epsilon^2}{ (\beta(f^*)^\prime_+(-\tilde{\nu}))^2}\right)+\sum_{y\in\mathcal{Y}} e^{-2n(p_y-p_0)^2}
\,,\notag
\end{align}
where $p_y$ is the probability of the class $y\in\mathcal{Y}$, $p_0$ and $\tilde{\nu}$ are appropriately chosen constants, and $R_n$ and $\widetilde{R}_n$ depend on the complexity of the objection function class and    approach $0$ as $n\to\infty$ under appropriate assumptions;   see  \eqref{R_n_def}  and \eqref{eq:Rn_tilde_def} respectively for details.  Similarly to the OT-DRO case, in Theorem \ref{thm:OT_reg_DRO_conc} we also show that the corresponding ERM solution is an approximate solution to the population DRO problem with high probability.  These results  apply to OT cost functions of the form \eqref{eq:c_psi_delta_def} and to a general class of $f$-divergences, including the oft-used KL and $\alpha$-divergences.

\section{Concentration Inequalities for OT-DRO}\label{sec:OT_DRO_proof} We now proceed to the detailed proofs of our results. We start by considering OT-DRO, without   mixing with an $f$-divergence, focusing on  OT cost functions of the form \eqref{eq:c_psi_delta_def} which are not covered by previous works; in addition to being of independent interest, this will provide an important tool for our subsequent study of OT-regularized $f$-divergences.     In the following, we codify our assumptions regarding the objective and OT cost functions.
\begin{assumption}\label{assump:c_L}
Assume the following:
\begin{enumerate}
    \item   $\mathcal{Z}=\mathcal{X}\times\mathcal{Y}$, where    $\mathcal{Y}$ is Polish   
      and $\mathcal{X}\subset\mathbb{R}^d$ is convex  and Polish  (e.g., either closed or open); note that this makes $\mathcal{Z}$ a Polish space.
    \item $\|\cdot\|_{\mathcal{X}}$ is a norm on $\mathbb{R}^d$.
    \item $\mathcal{L}_\theta:\mathcal{X}\times\mathcal{Y}\to\mathbb{R}$, $\theta\in\Theta$, are measurable and we have $L_{\mathcal{X}}\in(0,\infty)$ such that $x\mapsto\mathcal{L}_\theta(x,y)$ is $L_{\mathcal{X}}$-Lipschitz with respect to $\|\cdot\|_{\mathcal{X}}$ for all $y$, $\theta$.
    \item  $c$ is a cost function on $\mathcal{Z}$ and we have $\delta\geq 0$ and $\psi:[0,\infty)\to[0,\infty]$ with $\psi(0)=0$  such that
\begin{align}\label{eq:c_bound_assump}
c_{\psi,\delta}\leq  c\leq c_\delta\,,
\end{align}
where $c_\delta$ is the cost function defined in \eqref{eq:PGD_cost} and $c_{\psi,\delta}$ was defined in \eqref{eq:c_psi_delta_def}.
\end{enumerate}

\end{assumption}
\begin{remark}
    Note that $c_\delta$, \eqref{eq:PGD_cost}, corresponds to the choice $\psi(t)=\infty 1_{t>0}$.
\end{remark}

A key challenge in the study of statistical guarantees on the OT-DRO problem  via its dual formulation \eqref{eq:OT_DRO},  is to effectively handle the non-compact interval in the  optimization over $\lambda$.
We start by deriving a bound that can be used to show the $c$-transformed loss has a well-behaved limit as $\lambda\to\infty$.
 \begin{lemma}\label{lemma:c_transform_limit}
Under Assumption \ref{assump:c_L} we have
\begin{align}
\|\mathcal{L}^c_{\theta,\lambda} -\mathcal{L}_\theta^{c_\delta}\|_\infty\leq \lambda\psi^*(L_{\mathcal{X}}/\lambda)
\end{align}
for all $\lambda>0$, $\theta\in\Theta$, where
\begin{align}
  \psi^*(s)\coloneqq\sup_{t\geq 0}\{st-\psi(t)\}
\end{align}
is the Legendre transform of $\psi$.
\end{lemma}
\begin{remark}
Note that the $c_\delta$-transformed loss $\mathcal{L}_{\theta,\lambda}^{c_\delta}$ does not depend on $\lambda$ and if $\delta=0$ then $\phi_{\psi,\delta}=\psi$ and $\mathcal{L}_\theta^{c_\delta}=\mathcal{L}_\theta$. Also note that if $\psi$ is LSC   then $c_{\psi,\delta}$ is LSC and  hence is a valid OT cost function, in which case this result can be applied   $c=c_{\psi,\delta}$.
\end{remark}
\begin{proof}
 In the following  we will suppress the $\theta$ dependence of the optimization objective, as it is not relevant to the computations. For $z=(x,y)\in\mathcal{Z}$, first compute
\begin{align}
   &\mathcal{L}^c_\lambda(z)   = \sup_{\tilde z\in\mathcal{Z}}\{ \mathcal{L} (\tilde z) -  \lambda c(z,\tilde z)\}\leq \sup_{\tilde z\in\mathcal{Z}}\{ \mathcal{L} (\tilde z) -  \lambda c_{\psi,\delta}(z,\tilde z)\}\label{eq:L_c_lambda_max_bound}\\
=&\max\left\{\sup_{\substack{\tilde x\in\mathcal{X}:\\\|\tilde{x}-x\|_{\mathcal{X}}\leq \delta}}\{ \mathcal{L} (\tilde x,y) -  \lambda \phi_{\psi,\delta}(\|\tilde{x}-x\|_{\mathcal{X}})\},\sup_{\substack{\tilde x\in\mathcal{X}:\\\|\tilde{x}-x\|_{\mathcal{X}}> \delta}}\{ \mathcal{L} (\tilde x,y) -  \lambda \phi_{\psi,\delta}(\|\tilde{x}-x\|_{\mathcal{X}})\}\right\}\notag\\
=&\max\left\{\mathcal{L}^{c_\delta}(z),\sup_{\tilde x\in\mathcal{X}:\|\tilde{x}-x\|_{\mathcal{X}}> \delta}\{ \mathcal{L} (\tilde x,y) -  \lambda \phi_{\psi,\delta}(\|\tilde{x}-x\|_{\mathcal{X}})\}\right\}\,.\notag
\end{align}

 The assumption that $\mathcal{L}(x,y)$ is $L_{\mathcal{X}}$-Lipschitz in $x$ implies that $\mathcal{L}^{c_\delta}(z)$ is finite. Together with \eqref{eq:L_c_lambda_max_bound} and  the bound $c\leq c_\delta$, this allows us to obtain
\begin{align}
0\leq   \mathcal{L}^c_\lambda(z)-\mathcal{L}^{c_\delta}(z)\leq &\max\left\{0,\sup_{\tilde x\in\mathcal{X}:\|\tilde{x}-x\|_{\mathcal{X}}> \delta}\{ \mathcal{L} (\tilde x,y) -  \lambda \phi_{\psi,\delta}(\|\tilde{x}-x\|_{\mathcal{X}})\}-\mathcal{L}^{c_\delta}(z)\right\}\,\,.\notag
\end{align}

By convexity of $\mathcal{X}$, for any $\tilde{x}\in\mathcal{X}$ satisfying $\|\tilde{x}-x\|_{\mathcal{X}}>\delta$ there exists $x^\prime\in\mathcal{X}$ that satisfies $\|{x}^\prime-x\|_{\mathcal{X}}=\delta$ and $\|\tilde{x}-x\|_{\mathcal{X}}=\delta+\|\tilde{x}-x^\prime\|_{\mathcal{X}}$.  Therefore we  can compute
\begin{align}
    \mathcal{L}(\tilde{x},y)-\lambda\phi_{\psi,\delta}(\|\tilde{x}-x\|_{\mathcal{X}})-\mathcal{L}^{c_\delta}(x,y)
    \leq&\mathcal{L}(\tilde{x},y)-\lambda\phi(\|\tilde{x}-x\|_{\mathcal{X}})-\mathcal{L}(x^\prime,y)\\
    \leq&L_{\mathcal{X}}\|\tilde{x}-x^\prime\|_{\mathcal{X}}-\lambda\phi(\delta+\|\tilde{x}-x^\prime\|_{\mathcal{X}})\notag\\
    =&L_{\mathcal{X}}\|\tilde{x}-x^\prime\|_{\mathcal{X}}-\lambda\psi(\|\tilde{x}-x^\prime\|_{\mathcal{X}})\notag\\
    \leq&\sup_{t\geq 0}\{L_{\mathcal{X}}t-\lambda\psi(t)\}=\lambda\psi^*(L_{\mathcal{X}}/\lambda)\,.\notag
\end{align}
Maximizing over $\{\tilde{x}\in\mathcal{X}:\|\tilde{x}-x\|_{\mathcal{X}}>\delta\}$ we see that
\begin{align}
0\leq\mathcal{L}^c_\lambda(z) -\mathcal{L}^{c_\delta}(z)\leq \max\{0,\lambda\psi^*(L_{\mathcal{X}}/\lambda)\}=\lambda\psi^*(L_{\mathcal{X}}/\lambda)\,, 
\end{align}
where the  equality follows from the fact that $\psi(0)=0$. Maximizing over   $z\in\mathcal{Z}$ completes the proof.
\end{proof}

Next we use  Lemma \ref{lemma:c_transform_limit} to derive a covering number bound that is able to handle the unbounded domain for $\lambda$. To do so, we make the following assumptions:
\begin{assumption}\label{assumption:psi_star}
In addition to Assumption \ref{assump:c_L}, suppose that:    \begin{enumerate}
        \item There exists $M\in[0,\infty)$ such that  for all $(x,y)\in\mathcal{Z}$ we have
            \begin{align} 
                \sup_{\tilde{x}\in D_{x,y}}    c((x,y),(\tilde{x},y))\leq M\,,\,\,\,D_{x,y}\coloneqq\{\tilde{x}\in \mathcal{X}: c((x,y),(\tilde{x},y))<\infty\}\,,
            \end{align}  \label{assump:c_bound}
        \item  $\psi^*(t)<\infty$ for all $t>0$,
        \item $\psi^*(t)=o(t)$   as $t\to 0^+$. 
    \end{enumerate}
\end{assumption}
\begin{lemma}\label{lemma:covering_number_OT_DRO}
Suppose that Assumption  \ref{assumption:psi_star} is satisfied.
Define 
\begin{align}\label{eq:G_c_def}
\mathcal{G}\coloneqq \{\mathcal{L}_\theta:\theta\in\Theta\}\,,\,\,\,\,
\mathcal{G}_c\coloneqq \{\mathcal{L}_{\theta,\lambda}^c:\theta\in\Theta,\lambda\in(0,\infty)\}\cup \{\mathcal{L}_\theta^{c_\delta}:\theta\in\Theta\}\,.
\end{align}
Then for all $\epsilon_1,\epsilon_2>0$ we have the following relation between covering numbers in the supremum norm:
\begin{align}\label{eq:covering_number_bound}
N(\epsilon_1+\epsilon_2,\mathcal{G}_c,\|\cdot\|_\infty)\leq\left(\left\lceil \frac{M\lambda_*(\epsilon_2)}{2\epsilon_2}\right\rceil+1\right) N(\epsilon_1,\mathcal{G},\|\cdot\|_\infty)\,,
\end{align}
where
\begin{align}\label{eq:lambda_star_def}
\lambda_*(\epsilon_2)\coloneqq\inf\{\lambda>0: \lambda\psi^*(L_{\mathcal{X}}/\lambda)\leq \epsilon_2\}\,.
\end{align}
We also note that $\lambda_*(\epsilon_2)$ is finite.

\end{lemma}
\begin{proof}
If $N(\epsilon_1,\mathcal{G},\|\cdot\|_\infty)=\infty$ the result is trivial, so suppose it is finite. Lemma \ref{lemma:c_transform_limit} implies that for all $\theta\in\Theta,\lambda>0$ we have
\begin{align}\label{eq:L_psi_star_bound}
\|\mathcal{L}^c_{\theta,\lambda} -\mathcal{L}_\theta^{c_\delta}\|_\infty\leq \lambda\psi^*(L_{\mathcal{X}}/\lambda)\,.
\end{align}
Combined with the assumption that $\psi^*(t)<\infty$ for all $t>0$ we see that $\mathcal{L}^c_{\theta,\lambda}$ is real-valued (recall that the Lipschitz assumption on $\mathcal{L}_\theta$ implies that $\mathcal{L}_\theta^{c_\delta}$ is real-valued).  The assumption \eqref{eq:c_bound_assump} implies
\begin{align}
\mathcal{L}^c_{\theta,\lambda}(x,y)=\sup_{\tilde{x}\in D_{x,y}}\{\mathcal{L}_\theta(\tilde{x},y)-\lambda c((x,y),(\tilde{x},y))\}\,,
\end{align}
therefore
\begin{align}\label{eq:Delta_L_lambda_bound}
\|\mathcal{L}^c_{\theta,\lambda_1}-\mathcal{L}^c_{\theta,\lambda_2}\|_\infty\leq \sup_{(x,y)\in\mathcal{Z}}\sup_{\tilde{x}\in D_{x,y}}|\lambda_2 c((x,y),(\tilde{x},y))-\lambda_1 c((x,y),(\tilde{x},y))|\leq M|\lambda_2-\lambda_1|\,,
\end{align}
and, by similar computations,
\begin{align}\label{eq:Delta_L_theta_bounds}
\|\mathcal{L}^c_{\theta_1,\lambda}-\mathcal{L}^c_{\theta_2,\lambda}\|_\infty\leq \|\mathcal{L}_{\theta_1}-\mathcal{L}_{\theta_2}\|_\infty\,,\,\,\,\,
\|\mathcal{L}_{\theta_1}^{c_\delta}-\mathcal{L}_{\theta_2}^{c_\delta}\|_\infty\leq\|\mathcal{L}_{\theta_1}-\mathcal{L}_{\theta_2}\|_\infty\,.
\end{align}

With $\lambda_*$ defined by \eqref{eq:lambda_star_def}, we note that the assumption $\psi^*(t)=o(t)$   as $t\to 0^+$ implies that  $\lambda_*(\epsilon_2)<\infty$ for all $\epsilon_2>0$. $\psi^*$ is convex and, by assumption, it is real-valued on $(0,\infty)$, therefore $\psi^*$ is continuous on $(0,\infty)$.  Hence either $\lambda_*(\epsilon_2)=0$ or $\lambda_*(\epsilon_2)\psi^*(L_{\mathcal{X}}/\lambda_*(\epsilon_2))\leq \epsilon_2$.   Define
\begin{align}
N(\epsilon_2)\coloneqq\left\lceil \frac{M\lambda_*(\epsilon_2)}{2\epsilon_2}\right\rceil\in\mathbb{Z}_0\,.
\end{align}
First consider the case where $\lambda_*(\epsilon_2)>0$  (and hence $N(\epsilon_2)>0$) and define
\begin{align}
\lambda_j\coloneqq (j-1/2)\frac{\lambda_*(\epsilon_2)}{N(\epsilon_2)}\,,\,\,\, j=1,...,N(\epsilon_2)\,.
\end{align}

  Let $\mathcal{L}_{\theta_i}$, $i=1,...,N(\epsilon_1,\mathcal{G},\|\cdot\|_\infty)$ be a minimal $\epsilon_1$-cover of $\mathcal{G}$.  We will show that  $\mathcal{L}_{\lambda_j,\theta_i}^c,\mathcal{L}^{c_\delta}_{\theta_i}$, where $i\in\{1,...,N(\epsilon_1,\mathcal{G},\|\cdot\|_\infty)\}$, $j\in\{1,...,N(\epsilon_2)\}$, is an $\epsilon_1+\epsilon_2$-cover of $\mathcal{G}_c$:

Given $\theta\in\Theta, \lambda>0$, there exists $i$ such that
\begin{align}
\|\mathcal{L}_\theta-\mathcal{L}_{\theta_i}\|_\infty\leq \epsilon_1
\end{align}
and hence \eqref{eq:Delta_L_theta_bounds} implies
\begin{align}
\|\mathcal{L}_\theta^{c_\delta}-\mathcal{L}_{\theta_i}^{c_\delta}\|_\infty\leq\|\mathcal{L}_\theta-\mathcal{L}_{\theta_i}\|_\infty\leq \epsilon_1\,.
\end{align}
If $\lambda\geq \lambda_*(\epsilon_2)$ then, using the fact that  $\lambda\mapsto \lambda\psi^*(L_{\mathcal{X}}/\lambda)$ is non-increasing along with \eqref{eq:L_psi_star_bound}, we can compute
\begin{align}
\|\mathcal{L}_{\theta,\lambda}^c-\mathcal{L}_{\theta_i}^{c_\delta}\|_\infty\leq&\|\mathcal{L}_{\theta,\lambda}^c-\mathcal{L}_{\theta}^{c_\delta}\|_\infty+\|\mathcal{L}_{\theta}^{c_\delta}-\mathcal{L}_{\theta_i}^{c_\delta}\|_\infty\\
\leq&\lambda\psi^*(L_{\mathcal{X}}/\lambda)+\epsilon_1\notag\\
\leq& \lambda_*(\epsilon_2)\psi^*(L_{\mathcal{X}}/ \lambda_*(\epsilon_2))+\epsilon_1\leq \epsilon_2+\epsilon_1\,.\notag
\end{align}
If $\lambda<\lambda_*(\epsilon_2)$ then there exists $j\in\{1,...,N(\epsilon_2)\}$ such that 
\begin{align}
    \lambda\in [(j-1)\lambda_*(\epsilon_2)/N(\epsilon_2),j\lambda_*(\epsilon_2)/N(\epsilon_2)]
\end{align} and so we can use \eqref{eq:Delta_L_lambda_bound} and \eqref{eq:Delta_L_theta_bounds} to compute
\begin{align}
\|\mathcal{L}^c_{\theta,\lambda}-\mathcal{L}^c_{\theta_i,\lambda_j}\|_\infty\leq &\|\mathcal{L}^c_{\theta,\lambda}-\mathcal{L}^c_{\theta_i,\lambda}\|_\infty+\|\mathcal{L}^c_{\theta_i,\lambda}-\mathcal{L}^c_{\theta_i,\lambda_j}\|_\infty\\
\leq &\|\mathcal{L}_{\theta}-\mathcal{L}_{\theta_i}\|_\infty+M|\lambda-\lambda_j|\notag\\
\leq &\epsilon_1+M \frac{\lambda_*(\epsilon_2)}{2N(\epsilon_2)}\leq  \epsilon_1+\epsilon_2\,.\notag
\end{align}
This proves that $\mathcal{L}_{\lambda_j,\theta_i}^c,\mathcal{L}^{c_\delta}_{\theta_i}$, $i=1,...,N(\epsilon_1,\mathcal{G},\|\cdot\|_\infty)$, $j=1,...,N(\epsilon_2)$ is an $\epsilon_1+\epsilon_2$-cover of $\mathcal{G}_c$ and therefore $N(\epsilon_1+\epsilon_2,\mathcal{G}_c,\|\cdot\|_\infty)\leq (N(\epsilon_2)+1)N(\epsilon_1,\mathcal{G},\|\cdot\|_\infty)$ as claimed.

Thus we have proven the lemma under the assumption that $\lambda_*(\epsilon_2)>0$.  In the case where $\lambda_*(\epsilon_2)=0$, by making the obvious modifications to the above argument, one sees  that $\mathcal{L}^{c_\delta}_{\theta_i}$, $i=1,...,N(\epsilon_1,\mathcal{G},\|\cdot\|_\infty)$ provides the desired cover. This completes the proof.  
\end{proof}

Using the above lemmas, we now derive concentration inequalities for the OT-DRO problem. The following lists  the additional conditions we will require.
\begin{assumption}\label{assump:conc_ineq}
In addition to Assumption  \ref{assumption:psi_star}, assume the following:
\begin{enumerate}
\item There exists $\beta\in(0,\infty)$ such that $0\leq \mathcal{L}_\theta\leq \beta$ for all $\theta\in\Theta$.\label{assump:L_unif_bound}
\item  $P\in\mathcal{P}(\mathcal{Z})$.
\item $N(\epsilon_1,\mathcal{G},\|\cdot\|_\infty)<\infty$ for all $\epsilon_1>0$, where $\mathcal{G}\coloneqq \{\mathcal{L}_\theta:\theta\in\Theta\}$.
\item $h_i:(0,\infty)\to (0,\infty)$, $i=1,2$, are measurable and $h_1(\tilde{\epsilon})+h_2(\tilde{\epsilon})=\tilde{\epsilon}$ for all $\tilde{\epsilon}>0$.
\end{enumerate}
\end{assumption}

\begin{theorem}\label{thm:conc_ineq_min_value}
Under Assumption \ref{assump:conc_ineq}, for $n\in\mathbb{Z}^+$,  $\epsilon>0$, $r>0$,  we have
\begin{align}
&{P}^n\left(\pm\left(\inf_{\theta\in\Theta}\sup_{Q: C(P,Q)\leq r}E_Q[\mathcal{L}_\theta]-\inf_{\theta\in\Theta}\sup_{Q: C(P_{n},Q)\leq r}E_Q[\mathcal{L}_\theta]\right)\geq 2D_n+\epsilon\right)\\
\leq&\exp\left(-\frac{2\epsilon^2n}{\beta^2}\right)\,,\notag
\end{align}
where $P_{n}\coloneqq\frac{1}{n}\sum_{i=1}^n \delta_{z_i}$ denotes the empirical measure corresponding to $z\in\mathcal{Z}^n$ and 
\begin{align}\label{eq:D_n_def}
 D_n   \coloneqq& 12n^{-1/2}\int_0^{\beta}\sqrt{\log \left(\left(\left\lceil \frac{M\lambda_*(h_2(\tilde{\epsilon}))}{2h_2(\tilde{\epsilon})}\right\rceil+1\right) N(h_1(\tilde{\epsilon}),\mathcal{G},\|\cdot\|_\infty)\right)}d\tilde{\epsilon}\,.
\end{align}
\end{theorem}
\begin{remark}
To obtain this result with $D_n$ given by \eqref{eq:D_n_def} we use Dudley's entropy integral over $(0,\infty)$. We restrict our attention to this case for simplicity of the presentation, though one can easily obtain a more general result using the entropy integral over $(\tilde{\epsilon}_0,\infty)$, $\tilde{\epsilon}_0>0$, as in, e.g., Theorem 5.22 of \cite{wainwright2019high}, in order to handle   families of objective functions for which the entropy integral over $(0,\beta)$ diverges.
\end{remark}
\begin{proof}
 The assumed bound  $c\leq c_\delta$ implies  $c$ is zero on the diagonal.  $\mathcal{X}$ and $\mathcal{Y}$ were assumed to be Polish, hence $\mathcal{Z}$ is Polish. Therefore Proposition  \ref{prop:OT_DRO} implies 
\begin{align}
&\sup_{Q: C(P,Q)\leq r}E_Q[\mathcal{L}_\theta]
=\inf_{\lambda>0}\{\lambda r+E_{P}[ \mathcal{L}_{\theta,\lambda}^c]\}\,,\label{eq:duality_P}\\
&\sup_{Q: C(P_n,Q)\leq r}E_Q[\mathcal{L}_\theta]
=\inf_{\lambda>0}\{\lambda r+E_{P_n}[ \mathcal{L}_{\theta,\lambda}^c]\}\,.\label{eq:duality_Pn}
\end{align} 
The assumed bounds on $\mathcal{L}_\theta$ imply $0\leq \mathcal{L}^c_{\theta,\lambda}\leq \beta$ for all $\theta,\lambda$. In particular, $\mathcal{L}^c_{\theta,\lambda}\in L^1({P})\cap L^1(P_n)$ and both \eqref{eq:duality_P} and \eqref{eq:duality_Pn} are finite, as are their infima over $\theta\in\Theta$. Therefore we can compute
\begin{align}\label{eq:min_value_bound}
&\pm\left(\inf_{\theta\in\Theta}\sup_{Q: C(P,Q)\leq r}E_Q[\mathcal{L}_\theta]-\inf_{\theta\in\Theta}\sup_{Q: C(P_n,Q)\leq r}E_Q[\mathcal{L}_\theta]\right)\\
=&\pm\left(\inf_{\theta\in\Theta,\lambda>0}\{\lambda r+E_{{P}}[ \mathcal{L}_{\theta,\lambda}^c]\}-\inf_{\theta\in\Theta,\lambda>0}\{\lambda r+E_{P_n}[ \mathcal{L}_{\theta,\lambda}^c]\}\right)\notag\\
\leq&\sup_{g\in\mathcal{G}_c}\left\{\pm\left(E_{{P}}[g]-\frac{1}{n}\sum_{i=1}^ng(z_i)\right)\right\}\coloneqq \phi_\pm(z)\,,\notag
\end{align}
where $\mathcal{G}_c$ was defined in \eqref {eq:G_c_def}. Note that $\mathcal{G}_c$ consists of universally measurable functions, $0\leq g\leq \beta$ for all $g\in \mathcal{G}_c$, and Lemma \ref{lemma:covering_number_OT_DRO} implies that   the following relation between covering numbers of $\mathcal{G}$ and $\mathcal{G}_c$ in the supremum norm for all $\tilde{\epsilon}>0$:
\begin{align}\label{eq:covering_number_bound_h}
N(\tilde \epsilon,\mathcal{G}_c,\|\cdot\|_\infty)\leq\left(\left\lceil \frac{M\lambda_*(h_2(\tilde{\epsilon}))}{2h_2(\tilde{\epsilon})}\right\rceil+1\right) N(h_1(\tilde{\epsilon}),\mathcal{G},\|\cdot\|_\infty)<\infty\,.
\end{align}

Now apply the uniform law of large numbers result from, e.g., Theorem 3.3, Eq. (3.8) -
(3.13) in \cite{mohri2018foundations} (this reference assumes $[0, 1]$-valued functions but the result can be shifted
and scaled to apply to any set of uniformly bounded functions), to   obtain
\begin{align}
E_{{P}^n}\left[\phi_\pm\right]\leq 2\mathcal{R}_{\mathcal{G}_c,{P},n}
\end{align}
for all  $n\in\mathbb{Z}^+$. The Rademacher complexity, $\mathcal{R}_{\mathcal{G}_c,{P},n}$, can be bounded using the Dudley entropy integral; see, e.g., Corollary 5.25 in \cite{vanHandel}:
\begin{align}
\mathcal{R}_{\mathcal{G}_c,{P},n}\leq&12n^{-1/2}\int_0^{\beta}\sqrt{\log N(\tilde{\epsilon},\mathcal{G}_c,\|\cdot\|_\infty)}d\tilde{\epsilon}\,.
\end{align} 
Combining this with the covering number bound \eqref{eq:covering_number_bound_h} we obtain
\begin{align}\label{eq:ULLN_mean_bound}
&E_{{P}^n}\left[\phi_\pm\right]\leq 2D_n\,.
\end{align}

Finally, applying McDiarmid's inequality (see, e.g., Theorem D.8 in \cite{mohri2018foundations}) to $\phi_\pm$ and combining this with \eqref{eq:min_value_bound} and \eqref{eq:ULLN_mean_bound} 
we can compute
\begin{align}
&{P}^n\left(\pm\left(\inf_{\theta\in\Theta}\sup_{Q: C(P,Q)\leq r}E_Q[\mathcal{L}_\theta]-\inf_{\theta\in\Theta}\sup_{Q: C(P_{n},Q)\leq r}E_Q[\mathcal{L}_\theta]\right)\geq 2D_n+\epsilon\right)\\
\leq&{P}^n\left(\phi_\pm-E_{{P}^n}\left[\phi_\pm\right]\geq \epsilon\right)
\leq\exp\left(-\frac{2\epsilon^2n}{\beta^2}\right)\,.\notag
\end{align}
\end{proof}
\subsection{OT-DRO: ERM Bound}
By a similar argument, we can also show that a solution to the empirical OT-DRO problem (i.e., the ERM solution; see \eqref{eq:OT_ERM} below) is also an approximate solution to the population DRO problem with high probability (also allowing for an optimization error tolerance, $\epsilon_n^{\text{opt}}$, and optimization failure probability, $\delta_n^{\text{opt}}$).
\begin{theorem}\label{thm:OT_DRO_emp_optimizer}
Under Assumption \ref{assump:conc_ineq},  suppose we have $r>0$, $\theta_{*,n}:\mathcal{Z}^n\to \Theta$, $\epsilon_n^{\text{opt}}\geq 0$, ${\delta}^{\text{opt}}_n\in[0,1]$, and $E_n\subset \mathcal{Z}^n$ such that ${P}^n(E^c_n)\leq  {\delta}^{\text{opt}}_n$  and
\begin{align}\label{eq:OT_ERM}
\sup_{Q: C(P_{n},Q)\leq r}E_Q[\mathcal{L}_{\theta_{*,n}}]\leq \inf_{\theta\in\Theta}\sup_{Q: C(P_{n},Q)\leq r}E_Q[\mathcal{L}_\theta]+\epsilon_n^{\text{opt}}
\end{align}
on $E_n$, where $P_{n}\coloneqq\frac{1}{n}\sum_{i=1}^n \delta_{z_i}$.

 Then for $n\in\mathbb{Z}^+$, $\epsilon>0$, and with   $D_n$  as defined in \eqref{eq:D_n_def},    we have
\begin{align}\label{eq:conc_ineq_empirical_sol}
&{P}^n\left( \sup_{Q: C(P,Q)\leq r}E_Q[\mathcal{L}_{\theta_{*,n}}]\geq \inf_{\theta\in\Theta}\sup_{Q: C(P,Q)\leq r}E_Q[\mathcal{L}_\theta]+4D_n+\epsilon^{\text{opt}}_n+\epsilon\right)\\
\leq& \exp\left(-\frac{\epsilon^2n}{ 2\beta^2}\right)+{\delta}^{\text{opt}}_n\,,\notag
\end{align}
\end{theorem}
\begin{remark}
In the absence of sufficient measurability assumptions,  $P^n$ in \eqref{eq:conc_ineq_empirical_sol} and in the following proof should be interpreted as   the outer probability.
\end{remark}
\begin{proof}
On $E_n$ we have
\begin{align}
    &\sup_{Q: C(P,Q)\leq r}E_Q[\mathcal{L}_{\theta_{*,n}}]-\inf_{\theta\in\Theta}\sup_{Q: C(P,Q)\leq r}E_Q[\mathcal{L}_\theta]\\
    \leq&\sup_{Q: C(P,Q)\leq r}E_Q[\mathcal{L}_{\theta_{*,n}}]-\sup_{Q: C(P_{n},Q)\leq r}E_Q[\mathcal{L}_{\theta_{*,n}}]+\inf_{\theta\in\Theta}\sup_{Q: C(P_{n},Q)\leq r}E_Q[\mathcal{L}_\theta]+\epsilon_n^{\text{opt}}\notag\\
    &-\inf_{\theta\in\Theta}\sup_{Q: C(P,Q)\leq r}E_Q[\mathcal{L}_\theta]\notag\\
    \leq&\sup_{\theta\in\Theta}\{F_\theta-F_{n,\theta}\}+\sup_{\theta\in\Theta}\{F_{n,\theta}-F_\theta\}+\epsilon_n^{\text{opt}}\,,\notag
\end{align}
where
\begin{align}
F_{n,\theta}\coloneqq& \sup_{Q: C(P_{n},Q)\leq r}E_Q[\mathcal{L}_\theta]=\inf_{\lambda>0}\{\lambda r+E_{P_{n}}[\mathcal{L}^c_{\theta,\lambda}]\}\,,\\
F_\theta\coloneqq &\sup_{Q: C(P,Q)\leq r}E_Q[\mathcal{L}_\theta]=\inf_{\lambda>0}\{\lambda r+E_{{P}}[\mathcal{L}^c_{\theta,\lambda}]\}\,.\notag
\end{align}
Therefore
\begin{align}
&{P}^n\left( \sup_{Q: C(P,Q)\leq r}E_Q[\mathcal{L}_{\theta_{*,n}}]\geq \inf_{\theta\in\Theta}\sup_{Q: C(P,Q)\leq r}E_Q[\mathcal{L}_\theta]+4D_n+\epsilon^{\text{opt}}_n+\epsilon\right)\\
\leq& {\delta}^{\text{opt}}_n+{P}^n\left( \sup_{\theta\in\Theta}\left\{F_\theta-F_{n,\theta}\right\}+\sup_{\theta\in\Theta}\left\{F_{n,\theta}-F_\theta\right\}\geq 4 D_n+\epsilon \right)\notag\\
\leq& {\delta}^{\text{opt}}_n+{P}^n\left(\phi\geq 4D_n+\epsilon\right)\,,\notag
\end{align}
where
\begin{align}
\phi\coloneqq \sup_{g\in\mathcal{G}_c}\left\{E_{{P}}[g]-\frac{1}{n}\sum_{i=1}^ng(z_i)\right\}+\sup_{g\in\mathcal{G}_c}\left\{-\left(E_{{P}}[g]-\frac{1}{n}\sum_{i=1}^ng(z_i)\right)\right\}
\end{align}
and $\mathcal{G}_c$ was defined in \eqref{eq:G_c_def}.
By the same argument that lead to \eqref{eq:ULLN_mean_bound}  in the proof of Theorem \ref{thm:conc_ineq_min_value} we have
\begin{align}
&E_{{P}^n}\!\left[\phi\right]\leq 4D_n\,,
\end{align}
hence we can use McDiarmid's inequality to compute
\begin{align}
&{P}^n\left( \sup_{Q: C(P,Q)\leq r}E_Q[\mathcal{L}_{\theta_{*,n}}]\geq \inf_{\theta\in\Theta}\sup_{Q: C(P,Q)\leq r}E_Q[\mathcal{L}_\theta]+4D_n+\epsilon^{\text{opt}}_n+\epsilon\right)\\
\leq& {\delta}^{\text{opt}}_n+{P}^n\left(\phi-E_{{P}^n}[\phi]\geq \epsilon\right)\leq{\delta}^{\text{opt}}_n+ \exp\left(-\frac{\epsilon^2n}{ 2\beta^2}\right)\,.\notag
\end{align}
\end{proof}

\begin{remark}\label{remark:r_dependence}
    Note that our approach yields bounds that do not depend on the neighborhood size, $r$; this is well-suited for applications to adversarial training, where $r$ is small, but fixed (i.e., does not depend on $n$).   Our results have several advantages over those obtained by previous approaches that considered the adversarial training setting \cite{lee2018minimax, an2021generalization,azizian2023exact,gao2024wasserstein}  (note that these only considered $p$-Wasserstein costs).  
    Specifically, compare \eqref{eq:conc_ineq_empirical_sol} with Theorem 2 in \cite{lee2018minimax}; the latter  has (in our notation) a $r^{-(p-1)}$ factor in one of the error terms, thus it behaves poorly when $r$ is small. The bounds in \cite{an2021generalization,gao2024wasserstein} do not approach $0$ as $n\to \infty$ due to effect of the fixed neighborhood size; see Theorems 3 and 4 in \cite{an2021generalization} and  Theorem 3 in \cite{gao2024wasserstein}.  Finally, we note that the results in \cite{azizian2023exact} require more restrictive assumptions on the objective function (see their Assumption 5) when $r$ is not decaying with $n$.

\end{remark}

\subsection{Examples}\label{sec:examples_OT_DRO}
Next we provide several cases of interest  where $\psi(0)=0$, $\psi^*$ is finite, and $\psi^*(t)=o(t)$ as $t\to 0^+$. We focus on soft-constraint relaxations,  \eqref{eq:c_psi_delta_def} with real-valued $\psi$, but we also show that our method easily handles the PGD-cost \eqref{eq:PGD_cost} as a special case.
\begin{enumerate}
    \item  $\psi(t)=\alpha t^q$ for some $\alpha>0$, $q>1$: In this case, for $s>0$ we have
\begin{align}
\psi^*(s)=\alpha(q-1)\left(\frac{s}{\alpha q}\right)^{q/(q-1)}
\end{align}
and for $\epsilon_2>0$ we have
\begin{align}\label{eq:lambda_star_case1}
    \lambda_*(\epsilon_2)=\alpha^{-1}(L_{\mathcal{X}}/q)^q(q-1)^{q-1}\epsilon_2^{-(q-1)}\,.
\end{align}
Substituting this into the covering number bound \eqref{eq:covering_number_bound}, we see that the contribution from the $\lambda$ parameter is a factor that scales like  $O(\epsilon_2^{-q})$; this can be thought of as an effectively $q$-dimensional contribution to the complexity of the function space.
\item $\psi(t)=\alpha t^q+\beta t$ for some $\alpha,\beta>0$, $q>1$: In this case, for $s>0$ we have
\begin{align}\label{eq:psi_case2}
\psi^*(s)=\begin{cases}
    \alpha(q-1)\left(\frac{s-\beta}{\alpha q}\right)^{q/(q-1)} &\text{ if }s>\beta\\
    0& \text{ if } 0<s\leq \beta
\end{cases}\,.
\end{align}
The fact that $\psi^*$ vanishes on $(0,\beta]$   implies the bound $\lambda_*(\epsilon_2)\leq L_{\mathcal{X}}/\beta$.  Hence Lemma \ref{lemma:c_transform_limit} gives $\mathcal{L}^c_{\theta,\lambda}=\mathcal{L}^{c_\delta}_\theta$ for all  $\lambda\geq L_{\mathcal{X}}/\beta$. The corresponding contribution of  $\lambda$ to the covering number bound \eqref{eq:covering_number_bound} scales like $O(\epsilon_2^{-1})$ and hence is effectively one-dimensional; this matches the fact that in such cases one can  restrict $\lambda$ to a compact interval.
\item $\psi(t)=\alpha (e^{qt}-1)$ for some $\alpha,q>0$: In this case, for $s>0$ we have
\begin{align}\label{eq:psi_case3}
\psi^*(s)=\begin{cases}
\frac{s}{q}\log\left(\frac{s}{\alpha q}\right)-\alpha\left(\frac{s}{\alpha q}-1\right) & \text{ if }s>\alpha q\\
0& \text{ if } 0<s\leq \alpha q
\end{cases}\,.
\end{align}
This example exhibits similar quantitative behavior to the previous case, making an essentially one-dimensional contribution to the covering number bound.
\item $\psi(t)=\infty 1_{t>0}$: This corresponds to the PGD case, \eqref{eq:PGD_cost}. Here we have $\psi_*(s)=0$ for all $s$, and hence $\lambda_*(\epsilon_2) =0$ for all $\epsilon_2>0$.  Thus  the covering number bound \eqref{eq:covering_number_bound} gives
\begin{align} 
N(\epsilon_1+\epsilon_2,\mathcal{G}_c,\|\cdot\|_\infty)\leq  N(\epsilon_1,\mathcal{G},\|\cdot\|_\infty)
\end{align}
for all $\epsilon_2>0$, consistent with the fact that  the minimization over $\lambda$ in \eqref{eq:OT_DRO} can be evaluated explicitly to give
\begin{align}
\inf_{\lambda>0}\{\lambda r+ E_{{P}}[\mathcal{L}^c_{\lambda}]\}=E_{(x,y)\sim P}\!\left[\sup_{\tilde{x}\in\mathcal{X}:\|x-\tilde{x}\|_{\mathcal{X}}\leq \delta}\mathcal{L}(\tilde{x},y)\right]\,.
\end{align}Thus our approach gracefully handles the classical PGD OT-cost as a special case.
    \end{enumerate} 
  The above examples all result in a modest increase in the covering number bound, apart from case (1) with very large $q$.  Thus, moving from  non-robust to  robust model training  does not lead to a significant weakening of the statistical error bounds.
 Corresponding explicit bounds on $D_n$, \eqref{eq:D_n_def}, are given in Appendix \ref{sm:example_detials_DN}.

\section{Concentration Inequalities for DRO with OT-Regularized $f$-Divergences} \label{sec:OT_Df_DRO_proof} 
In this section, we prove the concentration inequality \eqref{eq:OT_rec_conc_intro} for  DRO with OT-regularized $f$-divergences, working under the following assumptions; in particular, we now restrict our attention to classification problems (i.e., discrete $\mathcal{Y}$).
\begin{assumption}\label{asump:OT_reg_DRO}
In addition to Assumption \ref{assump:conc_ineq}, assume the following:
\begin{enumerate}
\item  $\mathcal{Y}$ is a finite set with cardinality $K\in\mathbb{Z}^+$, $K\geq 2$ (the label categories); we equip $\mathcal{Y}$ with the discrete metric, which makes $\mathcal{Z}\coloneqq \mathcal{X}\times\mathcal{Y}$ a Polish space.
\item The $\mathcal{Y}$-marginal distribution, $P_Y$, satisfies  $p_y\coloneqq P_{Y}(y)>0$ for all $y\in\mathcal{Y}$.\label{assump:Py>0}
    \item $\{\tilde{z}\in\mathcal{Z}:c(z,\tilde{z})<\infty\}=\mathcal{X}\times\{y\}$ for all $z=(x,y)\in\mathcal{Z}$.\label{assump:c_infty_set}
    \item We have $h_i:(0,\infty)\to(0,\infty)$, $i=1,2,3$ that are measurable and satisfy $h_1(\tilde{\epsilon})+h_2(\tilde{\epsilon})+h_3(\tilde{\epsilon})=\tilde{\epsilon}$ for all $\tilde{\epsilon}>0$.
    \item  We have $f\in\mathcal{F}_1(a,b)$, where $0\leq a<1<b\leq \infty$,  and $f$ is strictly convex on a neighborhood of $1$.\label{assump:f}
\item  $s_0\coloneqq f^\prime_+(1)\in \{f^*<\infty\}^o$. \label{assump:s0_def}
    \item $f^*$ is bounded below.
    \item    $\lim_{s\to\infty}(f^*)^\prime_+(s)=\infty$.\label{assump:f_star_prime_limit}
    \item $s(f^*)^\prime_+(-s)$ is bounded on $s\in[\tilde{\nu},\infty)$ for all $\tilde{\nu}\in\mathbb{R}$.\label{assump:f_star_prime_neg_infty}
\end{enumerate}
\end{assumption}
\begin{remark}
The assumption on $P_Y$ is trivial in practice, as one simply   removes unnecessary classes.  The assumption regarding the set where $c$ is finite holds for the OT cost functions $c_{\psi,\delta}$ \eqref{eq:c_psi_delta_def} with real-valued $\psi$.
    The assumptions on $f$ hold for  both KL and the $\alpha$-divergences for all $\alpha>1$; for details, see Section \ref{sec:examples_OT_reg_Df_DRO}.
\end{remark}

Specifically, we will prove the following result.
\begin{theorem}\label{thm:OT_reg_DRO_emp_objective_bound}
 Under Assumption \ref{asump:OT_reg_DRO}, let  $p_0\in(0,\min_y p_y)$ and $\tilde{\nu}\in\mathbb{R}$ such that 
 \begin{align}
(f^*)^\prime_+(-M-\tilde{\nu})\geq 1/p_0  \,.
 \end{align}  Then for all $\epsilon>0$, $n\in\mathbb{Z}^+$, $\lambda_n>0$ we have
\begin{align}
&P^n\left(\pm\left(\inf_{\theta\in\Theta}\sup_{Q:D_f^c(Q\|P)\leq r}E_Q[\mathcal{L}_\theta]-\inf_{\theta\in\Theta}\sup_{Q:D_f^c(Q\|P_n)\leq r}E_Q[\mathcal{L}_\theta]\right)\geq \max\{R_n,\widetilde{R}_n\}+\epsilon\right)\\
\leq&\exp\left(-\frac{2n\epsilon^2}{ \beta^2}\right)+
\exp\left(-\frac{2n\epsilon^2}{ (\beta(f^*)^\prime_+(-\tilde{\nu}))^2}\right)+\sum_{y\in\mathcal{Y}} e^{-2n(p_y-p_0)^2}
\,,\notag
\end{align}
where $R_n$ and $\widetilde{R}_n$ are defined below in \eqref{R_n_def}  and \eqref{eq:Rn_tilde_def} respectively.
\end{theorem}
\begin{remark}
For  a discussion of the order in $n$ of the error terms $R_n$ and $\widetilde{R}_n$, see  Section \ref{sec:examples_OT_reg_Df_DRO}.
\end{remark}

\subsection{Proof of Theorem \ref{thm:OT_reg_DRO_emp_objective_bound}}
The proof of Theorem \ref{thm:OT_reg_DRO_emp_objective_bound} is somewhat lengthy, and so we spread it over the following  subsections.
\subsubsection{Proof of Theorem \ref{thm:OT_reg_DRO_emp_objective_bound}:  Error Decomposition}\label{sec:OT_reg_error_decomp}
The primary new complication in the analysis of DRO with OT-regularized $f$-divergences (as compared to the OT-DRO case) is the $1/\lambda$ factor on the right-hand side of \eqref{eq:OT_f_div_DRO}. We note that in the Cressie-Read divergence case (without OT), the optimization over $\lambda$ can be simplified analytically; this technique was used in \cite{duchi2021learning} but is not applicable to the more general class of $f$-divergences studied in this work.  Thus the following analysis requires new  techniques.

To address the aforementioned complication, we start by decomposing the problem into two terms, covering different $\lambda$-domains, which we treat by different methods. First, by   a similar calculation to that in  the OT case, we have the error bound
\begin{align}
&\pm\left(\inf_{\theta\in\Theta}\sup_{Q:D_f^c(Q\|P)\leq r}E_Q[\mathcal{L}_\theta]-\inf_{\theta\in\Theta}\sup_{Q:D_f^c(Q\|P_n)\leq r}E_Q[\mathcal{L}_\theta]\right)\\
\leq&\sup_{\theta\in\Theta,\lambda>0}\left\{\pm\left(\inf_{\rho\in\mathbb{R}}\{\rho+\lambda E_{{P}}[f^*((\mathcal{L}^c_{\theta,\lambda}-\rho)/\lambda)]\}-\inf_{\rho\in\mathbb{R}}\{\rho+\lambda E_{{P_n}}[f^*((\mathcal{L}^c_{\theta,\lambda}-\rho)/\lambda)]\}\right)\right\}\notag\,.
\end{align}
Now, for all $\widetilde{D}_n\geq 0$ and all $\lambda_n>0$ (to be chosen later), we can use a union bound to obtain
\begin{align}
&P^n\left(\pm\left(\inf_{\theta\in\Theta}\sup_{Q:D_f^c(Q\|P)\leq r}E_Q[\mathcal{L}_\theta]-\inf_{\theta\in\Theta}\sup_{Q:D_f^c(Q\|P_n)\leq r}E_Q[\mathcal{L}_\theta]\right)\geq \widetilde{D}_n\right)\\
\leq&P^n\left(\sup_{\theta\in\Theta,\lambda\geq\lambda_n}\left\{\pm \left(\lambda\Lambda_f^P[\mathcal{L}^c_{\theta,\lambda}/\lambda]-\lambda\Lambda_f^{P_n}[\mathcal{L}^c_{\theta,\lambda}/\lambda]\right)\right\}\geq \widetilde{D}_n\right)\label{eq:f_div_term_lambda_infty}\\
&+P^n\left(\sup_{\theta\in\Theta,\lambda\in(0,\lambda_n)}\left\{\pm\left(\lambda\Lambda_f^P[\mathcal{L}^c_{\theta,\lambda}/\lambda]-\lambda\Lambda_f^{P_n}[\mathcal{L}^c_{\theta,\lambda}/\lambda]\right)\right\}\geq \widetilde{D}_n\right)\,,\label{eq:f_div_term_lambda_0}
\end{align}
where we changed variables in the inner infima to $\nu=\rho/\lambda$ and introduced the notation
\begin{align}\label{eq:Lambda_f_def}
\Lambda_f^Q[\phi]\coloneqq\inf_{\nu\in\mathbb{R}}\{\nu+E_Q[f^*(\phi-\nu)]\,,
\end{align}
which is defined so long as $\phi^-\in L^1(Q)$. Note that we have the identity
\begin{align}\label{eq:Lambda_shift} 
    \Lambda_f^Q[\phi+\gamma]=\gamma+\Lambda_f^Q[\phi]
\end{align}
for all $\gamma\in\mathbb{R}$.
In the KL case, $\Lambda_{KL}^Q[\phi]=\log E_Q[e^\phi]$ and so $\Lambda_f^Q[\phi]$ can  be thought of as a generalization of the cumulant generating function.

The term \eqref{eq:f_div_term_lambda_infty}, which involves the supremum over $\lambda\geq\lambda_n$, can be handled  in a  manner reminiscent of the OT case, as we show in the next subsection.  The supremum over $\lambda\in(0,\lambda_n)$ in \eqref{eq:f_div_term_lambda_0} must contend with the $1/\lambda$ singularity in the arguments to $\Lambda_f^P$ and $\Lambda_f^{P_n}$; this necessitates an alternative proof strategy. The key new tool is the following lemma, which provides conditions under which the optimization over $\nu$ in \eqref{eq:Lambda_f_def} can be restricted to a bounded domain. As this result is of  interest for applications beyond the current work, we prove it under more general conditions that those of Assumption \ref{asump:OT_reg_DRO}.
\begin{lemma}\label{lemma:nu_domain_bound_main}
 Let $(\Omega,\mathcal{M},Q)$ be a probability space and $\phi:\Omega\to\mathbb{R}$ be measurable with $\phi^-\in L^1(Q)$  and  $\phi\leq \beta$ for some  $\beta\in\mathbb{R}$. 

Let $f\in\mathcal{F}_1(a,b)$ with $a\geq 0$ and suppose $f$ is strictly convex in a neighborhood of $1$ and $s_0 \in\{f^*<\infty\}^o$ ($s_0$ as defined in part \ref{assump:s0_def} of Assumption \ref{asump:OT_reg_DRO}).

 Suppose we have $\tilde{\alpha}\in\mathbb{R}$, $p_{\tilde{\alpha}}\in(0,1)$, and $\nu_{\tilde{\alpha}}\in\mathbb{R}$ such that $\tilde{\alpha}-\nu_{\tilde{\alpha}}\in\overline{\{f^*<\infty\}}$,  $Q(\phi\geq \tilde\alpha) \geq p_{\tilde{\alpha}}$, and $(f^*)^\prime_+(\tilde{\alpha}-\nu_{\tilde{\alpha}})\geq 1/p_{\tilde{\alpha}}$. 
 
Then 
 $\nu_{\tilde{\alpha}}< \beta-s_0$ and
\begin{align}\label{eq:Lambda_f_P_bounded_nu_domain}
\Lambda_f^Q[\phi]=\inf_{\nu\in[\nu_{\tilde{\alpha}},\beta-s_0]} \{\nu+E_Q[f^*(\phi-\nu)]\}\,.
\end{align}

Moreover, if $\alpha\leq \phi\leq \beta$ then we have the simpler result 
\begin{align}\label{eq:Lambda_f_P_bounded_nu_domain2}
\Lambda_f^Q[\phi]=\inf_{\nu\in[\alpha-s_0,\beta-s_0]} \{\nu+E_Q[f^*(\phi-\nu)]\}\,.
\end{align}

\end{lemma}
\begin{remark}
The simpler distribution-independent case \eqref{eq:Lambda_f_P_bounded_nu_domain2} was previously obtained in Lemma 2.1 of \cite{birrell2025statistical}. The  distribution-dependent case \eqref{eq:Lambda_f_P_bounded_nu_domain}, which we believe to be new,  will be key for obtaining uniform (in $\lambda$) bounds on the terms in \eqref{eq:f_div_term_lambda_0}, thus mitigating the apparent difficulty stemming from the $1/\lambda$ singularity.
\end{remark}
\begin{proof}
 
Standard results in convex analysis imply that $f^*$ is  continuous on $\overline{\{f^*<\infty\}}$  and  the right-derivative $(f^*)^\prime_+$   is non-decreasing (and hence its definition can be naturally extended to $\overline{\{f^*<\infty\}}$). Also note that the assumption $a\geq 0$ implies $f^*$ is non-decreasing and hence $(f^*)^\prime_+\geq 0$.

Let $\nu<\nu_{\tilde{\alpha}}$.   First we show that
\begin{align}\label{eq:f_star_nu_lb_claim}
\nu+f^*(\phi-\nu)\geq \nu+f^*(\phi-\nu_{\tilde{\alpha}})+(f^*)^\prime_+(\tilde{\alpha}-\nu_{\tilde{\alpha}})(\nu_{\tilde{\alpha}}-\nu)1_{\phi\geq\tilde{\alpha}}\,.
\end{align}
Note that the claim is trivial if $f^*(\phi-\nu)=\infty$.  Therefore we suppose $\phi-\nu\in\{f^*<\infty\}$.  For $n\in\mathbb{Z}^+$ large enough we have $\phi-\nu_{\tilde{\alpha}}<\phi-\nu-1/n$,  $\phi-\nu_{\tilde{\alpha}},\phi-\nu-1/n\in\{f^*<\infty\}^o$ and hence we can use absolute continuity of $f^*$ compute
\begin{align}
\nu+f^*(\phi-\nu-1/n)=&\nu+f^*(\phi-\nu_{\tilde{\alpha}})+\int_{\phi-\nu_{\tilde{\alpha}}}^{\phi-\nu-1/n} (f^*)^\prime_+(s) ds\\
\geq&\nu+f^*(\phi-\nu_{\tilde{\alpha}})+(f^*)^\prime_+(\phi-\nu_{\tilde{\alpha}}) (\nu_{\tilde{\alpha}}-\nu-1/n) \notag\\
\geq&\nu+f^*(\phi-\nu_{\tilde{\alpha}})+(f^*)^\prime_+(\tilde{\alpha}-\nu_{\tilde{\alpha}}) (\nu_{\tilde{\alpha}}-\nu-1/n) 1_{\phi\geq \tilde{\alpha}}\,.\notag
\end{align}
Taking $n\to\infty$ and using continuity of $f^*$ on $\overline{\{f^*<\infty\}}$ we can  conclude the claimed bound.

Taking the expectation of both sides of \eqref{eq:f_star_nu_lb_claim}, which exist in $(-\infty,\infty]$ due to the assumption $\phi^-\in L^1(Q)$ and the fact that $f^*(t)\geq t$, we obtain
\begin{align}
\nu+E_Q[f^*(\phi-\nu)]\geq &\nu+E_Q[f^*(\phi-\nu_{\tilde{\alpha}})]+(f^*)^\prime_+(\tilde{\alpha}-\nu_{\tilde{\alpha}})(\nu_{\tilde{\alpha}}-\nu)Q(\phi\geq\tilde{\alpha})\\
\geq &\nu+E_Q[f^*(\phi-\nu_{\tilde{\alpha}})]+p_{\tilde{\alpha}}^{-1}(\nu_{\tilde{\alpha}}-\nu)Q(\phi\geq\tilde{\alpha})\notag\\
\geq  &\nu_{\tilde{\alpha}}+E_Q[f^*(\phi-\nu_{\tilde{\alpha}})]\,.\notag
\end{align}
Thus, for all $\nu<\nu_{\tilde{\alpha}}$, we have proven
\begin{align}\label{eq:Lambda_f_P_bounded_nu_domain_lb_app}
\nu+E_Q[f^*(\phi-\nu)]\geq \nu_{\tilde{\alpha}}+E_Q[f^*(\phi-\nu_{\tilde{\alpha}})]\,.
\end{align}

Now assume that    $f$ is strictly convex in a neighborhood of $1$ and 
    $s_0\coloneqq f_+^\prime(1)\in\{f^*<\infty\}^o$. These  assumptions imply  $f^*(s_0)=s_0$ and  $(f^*)^\prime_+(s_0)=1$ (see Lemma A.9 in \cite{JMLR:v23:21-0100}).   The bound \begin{align}
  (f^*)^\prime_+(\tilde{\alpha}-\nu_{\tilde\alpha})\geq 1/p_{\tilde\alpha}>1=(f^*)^\prime_+(s_0)      
    \end{align}
    implies $\tilde{\alpha}-\nu_{\tilde\alpha}> s_0$. 
    Also assuming that $\phi\leq \beta$, we have  $\tilde{\alpha}\leq \beta$ (otherwise $Q(\phi\geq \tilde{\alpha})=0$) and therefore $\beta-\nu_{\tilde\alpha}\geq \tilde\alpha-\nu_{\tilde\alpha}>s_0$.

Now let $\nu>\beta-s_0$. Noting that $(-\infty,s_0]\subset\{f^*<\infty\}^o$   we can  compute 
\begin{align}
f^*(\phi-(\beta-s_0))=&f^*(\phi-\nu)+\int_{\phi-\nu}^{\phi-(\beta-s_0)}(f^*)^\prime_+(s)ds\\
\leq&f^*(\phi-\nu)+(f^*)^\prime_+(s_0)(\nu-(\beta-s_0))= f^*(\phi-\nu)+\nu-(\beta-s_0)\,.\notag
\end{align}
Therefore
\begin{align}
\nu+E_Q[f^*(\phi-\nu)]\geq\beta-s_0+ E_Q[f^*(\phi-(\beta-s_0))]
\end{align}
for all $\nu>\beta-s_0$. 
Combining this with \eqref{eq:Lambda_f_P_bounded_nu_domain_lb_app} we arrive at \eqref{eq:Lambda_f_P_bounded_nu_domain}.  The simpler case \eqref{eq:Lambda_f_P_bounded_nu_domain2} follows similarly, and was also proven previously in Lemma 2.1 of \cite{birrell2025statistical}.
\end{proof}

\subsubsection{Proof of Theorem \ref{thm:OT_reg_DRO_emp_objective_bound}: $\lambda\geq\lambda_n$ Term}

To bound \eqref{eq:f_div_term_lambda_infty}, we start by showing that
\begin{align}\label{eq:Lambda_limit}
\lim_{\lambda\to\infty}\lambda \Lambda_f^P[\mathcal{L}^c_{\theta,\lambda}/\lambda]=E_P[\mathcal{L}_\theta^{c_\delta}]
\end{align}
(and similarly with $P_n$ replacing $P$),
with explicit error bounds.  The proof will utilize Lemma \ref{lemma:c_transform_limit}, but it will also require  new ingredients in order to handle the contribution from  $f^*$ and the infimum over $\nu$.  Note that we do not invoke the entirety of Assumption \ref{asump:OT_reg_DRO} in the following lemma, as proving \eqref{eq:Lambda_limit}    only requires a subset of those conditions.  
\begin{lemma}\label{lemma:f_div_lambda_infty}
In addition to Assumption \ref{assump:c_L}, part \ref{assump:L_unif_bound} of Assumption \ref{assump:conc_ineq},  and parts \ref{assump:f} - \ref{assump:s0_def} of Assumption \ref{asump:OT_reg_DRO}, assume that $\lambda_0>0$ satisfies $\beta/\lambda_0+s_0\in\{f^*<\infty\}^o$ ($\beta$ from part \ref{assump:L_unif_bound} of Assumption \ref{assump:conc_ineq} and $s_0$ from part \ref{assump:s0_def} of Assumption \ref{asump:OT_reg_DRO}).

Then for all $P\in\mathcal{P}(\mathcal{Z})$, $\lambda\geq \lambda_0$, $\theta\in\Theta$ we have
\begin{align}\label{eq:Lambda_f_P_infinity_bound}
&\left|\lambda\Lambda_f^P[\mathcal{L}^c_{\theta,\lambda}/\lambda]-E_{P}[\mathcal{L}_\theta^{c_\delta}]\right|
\leq\lambda\psi^*(L_{\mathcal{X}}/\lambda)+ \beta \left((f^*)^\prime_+(\beta /\lambda+s_0)-(f^*)^\prime_+(s_0)\right)\,.
\end{align}
\end{lemma}
\begin{remark}
All sufficiently large $\lambda_0$'s satisfy the required condition, due to the assumption that $s_0\in\{f^*<\infty\}^o$. Also note that, as $(f^*)^\prime_+$ is right continuous, if $\psi^*(t)=o(t)$ as $t\to 0^+$ then the upper    bound in  \eqref{eq:Lambda_f_P_infinity_bound} converges to $0$ as $\lambda\to \infty$.
\end{remark}
\begin{proof}
 Again  we  suppress the $\theta$ dependence, as it is not relevant to the computations.    
First use Lemma \ref{lemma:nu_domain_bound_main} (specifically, the simpler case \eqref{eq:Lambda_f_P_bounded_nu_domain2}) along a change of variables to rewrite
\begin{align}
\lambda\Lambda_f^P[\mathcal{L}^c_{\lambda}/\lambda]
=&\inf_{\eta\in[0,\beta]}\{\eta-\lambda s_0+\lambda E_P[f^*((\mathcal{L}^c_\lambda-\eta)/\lambda+ s_0)]\}\,.
\end{align}
For $\lambda\geq \lambda_0$ and $\eta\in[0,\beta]$ we have $(\mathcal{L}^c_\lambda-\eta)/\lambda)+s_0\leq \beta/\lambda_0+s_0$. As $f^*$ is non-decreasing, this implies $(\mathcal{L}^c_\lambda-\eta)/\lambda)+s_0\in\{f^*<\infty\}^o$.  Therefore    we can use  Taylor's formula for convex functions (see Theorem 1 in \cite{liese2006divergences}) together with the identities $f^*(s_0)=s_0$ and $(f^*)^\prime_+(s_0)=1$ (again, see Lemma A.9 in \cite{JMLR:v23:21-0100}) to obtain
\begin{align}
  f^*((\mathcal{L}^c_\lambda-\eta)/\lambda+ s_0)=s_0+ (\mathcal{L}^c_\lambda-\eta)/\lambda+R_{f^*}(s_0,(\mathcal{L}^c_\lambda-\eta)/\lambda+ s_0)
\end{align}
for all $\eta\in[0,\beta]$, where $R_{f^*}$ denotes the  remainder term in the expansion. Using this we obtain
\begin{align}
\lambda\inf_{\nu\in\mathbb{R}}\{\nu + E_{{P}}[f^*(\mathcal{L}^c_{\lambda}/\lambda-\nu]\}
=&E_P[ \mathcal{L}^c_\lambda]+\inf_{\eta\in[0,\beta]} E_P[\lambda R_{f^*}(s_0,(\mathcal{L}^c_\lambda-\eta)/\lambda+ s_0)]\,.\notag
\end{align}
Next, we recall that the remainder term is non-decreasing in $t\in[s,\infty)$   and satisfies
\begin{align}
    0\leq R_{f^*}(s,t)\leq |t-s||(f^*)^\prime_+(t)-(f^*)^\prime_+(s)|\,.
\end{align} 
This allows us to compute
\begin{align}
0\leq&\inf_{\eta\in[0,\beta]} E_P[\lambda R_{f^*}(s_0,(\mathcal{L}^c_\lambda-\eta)/\lambda+ s_0)]\leq  E_P[\lambda R_{f^*}(s_0,\mathcal{L}^c_\lambda/\lambda+ s_0)]\\
\leq& \lambda R_{f^*}(s_0,\beta /\lambda+ s_0)\notag\\
\leq&\beta\left((f^*)^\prime_+(\beta /\lambda+s_0)-(f^*)^\prime_+(s_0)\right)\,.\notag
\end{align}
Using the above together with Lemma \ref{lemma:c_transform_limit} we obtain
\begin{align}
&\left|\lambda\Lambda_f^P[\mathcal{L}^c_{\lambda}/\lambda]-E_{P}[\mathcal{L}^{c_\delta}]\right|\\
\leq&|E_P[ \mathcal{L}^c_\lambda]-E_{P}[\mathcal{L}^{c_\delta}]|+\inf_{\eta\in[0,\beta]} E_P[\lambda R_{f^*}(s_0,(\mathcal{L}^c_\lambda-\eta)/\lambda+ s_0)]\notag\\
\leq& \lambda\psi^*(L_{\mathcal{X}}/\lambda)+ \beta \left((f^*)^\prime_+(\beta /\lambda+s_0)-(f^*)^\prime_+(s_0)\right)\notag
\end{align}
as claimed.
\end{proof}

We are now ready to bound the term \eqref{eq:f_div_term_lambda_infty}. First apply Lemma \ref{lemma:f_div_lambda_infty} with $\lambda_0$ replaced by $\lambda_n$ (increasing in $n$, but with precise dependence on $n$ to be chosen later) to obtain
\begin{align}
&\sup_{\theta\in\Theta,\lambda\geq\lambda_n}\left\{\pm \left(\lambda\Lambda_f^P[\mathcal{L}^c_{\theta,\lambda}/\lambda]-\lambda\Lambda_f^{P_n}[\mathcal{L}^c_{\theta,\lambda}/\lambda]\right)\right\}\\
\leq&\sup_{\theta\in\Theta }\left\{\pm \left(E_P[\mathcal{L}^{c_\delta}_\theta]-E_{P_n}[\mathcal{L}^{c_\delta}_\theta]\right)\right\}+2\sup_{\lambda\geq\lambda_n} \{\lambda\psi^*(L_{\mathcal{X}}/\lambda)+ \beta \left((f^*)^\prime_+(\beta /\lambda+s_0)-(f^*)^\prime_+(s_0)\right)\}\notag\\
\leq&\sup_{\theta\in\Theta }\left\{\pm \left(E_P[\mathcal{L}^{c_\delta}_\theta]-E_{P_n}[\mathcal{L}^{c_\delta}_\theta]\right)\right\}+2\lambda_n\psi^*(L_{\mathcal{X}}/\lambda_n)+ 2\beta \left((f^*)^\prime_+(\beta /\lambda_n+s_0)-(f^*)^\prime_+(s_0)\right)\,.\notag
\end{align}

A straightforward application of McDiarmid's inequality gives
\begin{align}
    &P^n\left(\sup_{\theta\in\Theta }\left\{\pm \left(E_P[\mathcal{L}^{c_\delta}_\theta]-E_{P_n}[\mathcal{L}^{c_\delta}_\theta]\right)\right\}\geq E_{P^n}\left[\sup_{\theta\in\Theta }\left\{\pm \left(E_P[\mathcal{L}^{c_\delta}_\theta]-E_{P_n}[\mathcal{L}^{c_\delta}_\theta]\right)\right\}\right]+\epsilon\right)\\
    \leq& \exp\left(-\frac{2n\epsilon^2}{ \beta^2}\right)\notag
\end{align}
and, using Dudley's entropy integral together with the bound $
\|\mathcal{L}^{c_\delta}_{\theta_1}-\mathcal{L}^{c_\delta}_{\theta_2}\|_\infty\leq \|\mathcal{L}_{\theta_1}-\mathcal{L}_{\theta_2}\|$,
we can  compute
\begin{align}
E_{P^n}\left[\sup_{\theta\in\Theta }\left\{\pm \left(E_P[\mathcal{L}^{c_\delta}_\theta]-E_{P_n}[\mathcal{L}^{c_\delta}_\theta]\right)\right\}\right]\leq 24 n^{-1/2}\int_0^\beta \sqrt{\log N(\tilde{\epsilon},\mathcal{G},\|\cdot\|_\infty)} d\tilde{\epsilon}\,,
\end{align}
where $\mathcal{G}$ is as defined in \eqref{eq:G_c_def}.

Combining these we find
\begin{align}
&P^n\left(\sup_{\theta\in\Theta,\lambda\geq\lambda_n}\left\{\pm \left(\lambda\Lambda_f^P[\mathcal{L}^c_{\theta,\lambda}/\lambda]-\lambda\Lambda_f^{P_n}[\mathcal{L}^c_{\theta,\lambda}/\lambda]\right)\right\}\geq R_n+\epsilon\right)\leq  \exp\left(-\frac{2n\epsilon^2}{ \beta^2}\right)\,,
\end{align}
where
\begin{align}
R_n\coloneqq &2\lambda_n\psi^*(L_{\mathcal{X}}/\lambda_n)+ 2\beta \left((f^*)^\prime_+(\beta /\lambda_n+s_0)-(f^*)^\prime_+(s_0)\right)\label{R_n_def}\\
&+24 n^{-1/2}\int_0^\beta \sqrt{\log N(\tilde{\epsilon},\mathcal{G},\|\cdot\|_\infty)} d\tilde{\epsilon}\,.\notag
\end{align}

\subsubsection{Proof of Theorem \ref{thm:OT_reg_DRO_emp_objective_bound}: $\lambda<\lambda_n$ Term}\label{sec:OT_f_div_lambda_infinity}

Next we bound the term \eqref{eq:f_div_term_lambda_0}, which includes the $\lambda^{-1}$ singularity.  Our approach will be to show that, with high-probability, one can restrict the optimization over $\nu$ to a compact subset and that this results in a Lipschitz dependence on $\lambda$, despite the  apparent singularity as $\lambda\to 0^+$.

Start by combining part \ref{assump:c_bound} of Assumption \ref{assumption:psi_star} 
with part \ref{assump:c_infty_set} of Assumption \ref{asump:OT_reg_DRO} to see that
\begin{align}\label{eq:sup_c_tilde_def}
   \sup \tilde{c}\coloneqq \sup_{z,\tilde{z}:c(z,\tilde{z})<\infty} c(z,\tilde{z})\leq M<\infty   
\end{align}
and 
\begin{align}\label{eq:L_lambda_zero_limit}
\left|\mathcal{L}^c_{\theta,\lambda}(z)-\sup_{\tilde{z}:c(z,\tilde{z})<\infty}\mathcal{L}_\theta(\tilde{z})\right|\leq \lambda \sup \tilde{c}\,,
\end{align}
where $\sup_{\tilde{z}:c(z,\tilde{z})<\infty}\mathcal{L}_\theta(\tilde{z})=\sup_{\tilde{x}}\mathcal{L}_\theta(\tilde{x},y)$.

Define $\Delta\mathcal{L}^c_{\theta,\lambda}=\sup_{\mathcal{Z}}\mathcal{L}_\theta-\mathcal{L}^c_{\theta,\lambda}$ (note that this involves the supremum over all of $\mathcal{Z}$, not just over $\{\tilde{z}:c(z,\tilde{z})<\infty\}$).  As $\sup_{\mathcal{Z}}\mathcal{L}_\theta$ does not depend on $z$, \eqref{eq:Lambda_shift}  implies
\begin{align}
\lambda\Lambda_f^P[\mathcal{L}^c_{\theta,\lambda}/\lambda]=\sup_{\mathcal{Z}}\mathcal{L}_\theta+\lambda\Lambda_f^P[-\Delta\mathcal{L}^c_{\theta,\lambda}/\lambda]\,,
\end{align}
where $-\Delta\mathcal{L}^c_{\theta,\lambda}/\lambda\leq 0$.  To connect this with the bound \eqref{eq:L_lambda_zero_limit}, note that for all $z=(x,y)$ such that $y\in \text{argmax}_y \sup_{\tilde{x}}\mathcal{L}_\theta(\tilde{x},y)$ we have
\begin{align}
-\Delta\mathcal{L}^c_{\theta,\lambda}(z) =\mathcal{L}^c_{\theta,\lambda}(z)-\sup_{\tilde{z}:c(z,\tilde{z})<\infty}\mathcal{L}_\theta(\tilde{z})\leq \lambda \sup \tilde{c}\,,
\end{align}
and therefore
\begin{align}
    P(-\Delta\mathcal{L}^c_{\theta,\lambda}/\lambda\geq -\sup\tilde{c})\geq P_Y(\text{argmax}_y \sup_{\tilde{x}}\mathcal{L}_\theta(\tilde{x},y))\geq \min_y P_Y(y)\,.
\end{align}
Now fix  $p_0\in(0,\min_y P_Y(y))$. By part \ref{assump:f_star_prime_limit} of Assumption \ref{asump:OT_reg_DRO}, there exists $\tilde{\nu}\in\mathbb{R}$ such that $(f^*)^\prime_+(-M-\tilde{\nu})\geq 1/p_0$, and hence also $(f^*)^\prime_+(-\sup\tilde{c}-\tilde{\nu})\geq 1/p_0$; see \eqref{eq:sup_c_tilde_def}.   Lemma \ref{lemma:nu_domain_bound_main} then 
implies 
\begin{align}\label{eq:Lambda_P_rewrite}
\lambda\Lambda_f^P[-\Delta\mathcal{L}^c_{\theta,\lambda}/\lambda]=\inf_{\nu\in[{\tilde{\nu}},-s_0]} \{\lambda\nu+\lambda E_P[f^*(-\Delta\mathcal{L}^c_{\theta,\lambda}/\lambda-\nu)]\}\,.
\end{align}
Thus we have shown that  the infimum over $\nu$ can be restricted to a compact interval, uniformly in $\lambda$; this will be key for proving finite Rademacher complexity bounds on the corresponding family of functions.
 
We cannot immediately use the above argument to restrict the domain of the infimum over $\nu$ in the formula for $\lambda\Lambda_f^{P_n}[-\Delta\mathcal{L}^c_{\theta,\lambda}/\lambda]$ in a way that is uniform in $\lambda$ for all $z\in\mathcal{Z}^n$, as it is possible that  $z$ contains no sample whose label maximizes $\sup_{\tilde{x}}\mathcal{L}_\theta(\tilde{x},y)$; this case  presents a problem when $\lambda\to 0$, as do cases where there are not enough maximizing samples.  However,  the set of such  $z$'s   has probability that is exponentially decaying in $n$ and hence we can use a union bound to handle the low-probability exceptional set where the argument leading to \eqref{eq:Lambda_P_rewrite} fails for $P_n$.  To that end, for $y\in\mathcal{Y}$, $\xi_y\in(0,1)$ define
\begin{align}
 \tilde{\mathcal{Y}}_{y,\xi_y,n}\coloneqq\{\tilde{y}\in\mathcal{Y}^n: |\{i:\tilde{y}_i=y\}|\geq (1-\xi_y)p_y n\}\,,
\end{align}
where  $p_y\coloneqq P_Y(y)$,
and define
\begin{align}
\tilde{\mathcal{Y}}_{\xi,n}\coloneqq\cap_{y\in\mathcal{Y}} \tilde{\mathcal{Y}}_{y,\xi_y,n}\,.
\end{align}
A straightforward application of McDiarmid's inequality to the coordinate maps $Y_i:\mathcal{Y}^n\to\mathcal{Y}$ implies
\begin{align}
 P_Y^n( \tilde{\mathcal{Y}}_{y,\xi_y,n}^c)=&P_Y^n\left(\frac{1}{n}\sum_{i=1}^n1_y(Y_i)-p_y<-\xi_yp_y\right)\\
 \leq&e^{-2n\xi_y^2 p_y^2}\notag
\end{align}
and therefore,  letting  $\xi_y=1-p_0/p_y$ for all $y$,  we obtain
\begin{align}
 P^n( Y^n\in \tilde{\mathcal{Y}}_{\xi,n}^c)\leq &\sum_{y\in\mathcal{Y}} e^{-2n(p_y-p_0)^2}\,.
\end{align}
On the event $Y^n\in\tilde{\mathcal{Y}}_{\xi,n}$, by a similar argument to the population case, we then have 
\begin{align}
    P_n(-\Delta \mathcal{L}^c_{\theta,\lambda}/\lambda\geq -\sup\tilde{c})\geq P_n(\text{argmax}_y \sup_{\tilde{x}}\mathcal{L}_\theta)\geq \min_y\frac{1}{n}|\{i:Y_i=y\}|\geq \min_y(1-\xi_y)p_y= p_0\,,
\end{align}
and hence \eqref{eq:Lambda_shift}  and Lemma \ref{lemma:nu_domain_bound_main} together imply
\begin{align}
  \lambda\Lambda_f^{P_n}[  \mathcal{L}^c_{\theta,\lambda}/\lambda]= \sup_{\mathcal{Z}}\mathcal{L}_\theta+ \inf_{\nu\in[\tilde{\nu},-s_0]}\{\lambda\nu+\lambda E_{P_n}[f^*(-\Delta \mathcal{L}^c_{\theta,\lambda}/\lambda-\nu)]\}\,.
\end{align}

By combining the above results,  we obtain the following bound on \eqref{eq:f_div_term_lambda_0}:
\begin{align}\label{eq:union_bound_Y_tilde}
&P^n\left(\sup_{\theta\in\Theta,\lambda\in(0,\lambda_n)}\left\{\pm\left(\lambda\Lambda_f^P[\mathcal{L}^c_{\theta,\lambda}/\lambda]-\lambda\Lambda_f^{P_n}[\mathcal{L}^c_{\theta,\lambda}/\lambda]\right)\right\}\geq D\right)\\
\leq& P^n\left(\sup_{\theta\in\Theta,\lambda\in(0,\lambda_n)}\left\{\pm\left(\lambda\Lambda_f^P[\mathcal{L}^c_{\theta,\lambda}/\lambda]-\lambda\Lambda_f^{P_n}[\mathcal{L}^c_{\theta,\lambda}/\lambda]\right)\right\}\geq D,Y^n\in \tilde{\mathcal{Y}}_{\xi,n}\right)+\sum_{y\in\mathcal{Y}} e^{-2n(p_y-p_0)^2}\notag\\
\leq&P^n\left(\phi_\pm\geq D\right)+\sum_{y\in\mathcal{Y}} e^{-2n(p_y-p_0)^2}\,,\notag
\end{align}
where
\begin{align}
\phi_\pm\coloneqq& \sup_{\theta\in\Theta,\lambda\in(0,\lambda_n),\nu\in[\tilde{\nu},-s_0]}\left\{\pm\left(  E_{P}[\lambda f^*(-\Delta \mathcal{L}^c_{\theta,\lambda}/\lambda-\nu)]- E_{P_n}[\lambda f^*(-\Delta \mathcal{L}^c_{\theta,\lambda}/\lambda-\nu)]\right)\right\}\\
=&\sup_{g\in\mathcal{G}_{c,f}}\{\pm(E_{P_n}[g]-E_P[g])\}\,,\notag\\
    \mathcal{G}_{c,f}\coloneqq&\{g_{\theta,\lambda,\nu}:\theta\in\Theta,\lambda\in(0,\lambda_n),\nu\in[\tilde{\nu},-s_0]\}\,,\,\,\,g_{\theta,\lambda,\nu}\coloneqq \lambda(f^*(-\nu)-f^*(-\Delta \mathcal{L}^c_{\theta,\lambda}/\lambda-\nu))\,.
\end{align}

Supposing $z,z^\prime\in\mathcal{Z}^n$ differ only at index $j$, we have the bounded difference property
\begin{align}
|\phi_\pm(z)-\phi_\pm(z^\prime)|\leq   &\sup_{\theta\in\Theta,\lambda\in(0,\lambda_n),\nu\in[\tilde{\nu},-s_0]} \frac{\lambda}{n} \left| f^*(-\Delta \mathcal{L}^c_{\theta,\lambda}(z_j)/\lambda-\nu)]-  f^*(-\Delta \mathcal{L}^c_{\theta,\lambda}(z_j^\prime)/\lambda-\nu)\right|\\
\leq&\frac{1}{n}(f^*)^\prime_+(-\tilde{\nu})\sup_{\theta\in\Theta,\lambda\in(0,\lambda_n)}|\Delta \mathcal{L}^c_{\theta,\lambda}(z_j)-\Delta \mathcal{L}^c_{\theta,\lambda}(z_j^\prime)|\notag\\
\leq&\frac{\beta}{n}(f^*)^\prime_+(-\tilde{\nu}) \notag\,.
\end{align}
Therefore McDiarmid's inequality combined with the standard  bound on the mean in terms of the Rademacher complexity yields
\begin{align}
    P^n\left(\phi_\pm\geq 2\mathcal{R}_{\mathcal{G}_{c,f},P,n}+\epsilon\right)\leq \exp\left(-\frac{2n\epsilon^2}{ (\beta(f^*)^\prime_+(-\tilde{\nu}))^2}\right)\,.
\end{align}

To bound the Rademacher complexity, 
we note the following uniform and Lipschitz bounds on the functions $g_{\theta,\lambda,\nu}$ for all $\theta\in\Theta,\lambda\in(0,\lambda_n),\nu\in[\tilde{\nu},-s_0]$; for further details, see Appendix \ref{sm:Lipschitz}:
\begin{align}
  &0\leq g_{\theta,\lambda,\nu} \leq  \beta (f^*)^\prime_+(-\tilde{\nu}) \label{eq:G_cf_unif_bound}\,,\\
  &|g_{\theta_1,\lambda,\nu}-g_{\theta_2,\lambda,\nu}|\leq (f^*)^\prime_+(-\tilde{\nu})\|  \mathcal{L}_{\theta_1}- \mathcal{L}_{\theta_2}\|_\infty\,,\\
 &|g_{\theta,\lambda,\nu_1}-g_{\theta,\lambda,\nu_2}| \leq 2\lambda_n (f^*)^\prime_+(-\tilde{\nu})|\nu_1-\nu_2|\,,\\
 &\left|\partial_\lambda g_{\theta,\lambda,\nu}\right|
\leq f^*(-\tilde{\nu})-\inf f^*+\sup_{t\geq \tilde{\nu}}|t(f^*)^\prime_+(-t)|+(f^*)^\prime_+(-\tilde\nu)(\max\{-s_0,-\tilde{\nu}\}+\sup\tilde{c}  )\,.\label{eq:G_cf_lambda_deriv_bound}
 \end{align}
Regarding the derivative in $\lambda$,  note that $f^*$ is Lipschitz on $(-\infty,d)$ for all $d\in\mathbb{R}$ and $-\Delta \mathcal{L}^c_{\theta,\lambda}/\lambda-\nu$ is bounded above and is absolutely continuous when $\lambda$ is restricted to a compact interval due to the convexity of $\mathcal{L}^c_{\theta,\lambda}$ in $\lambda$.  Therefore $g_{\theta,\lambda,\nu}$ is absolutely continuous in $\lambda$ when restricted to compact intervals.  Hence the above a.s. bound on the derivative implies a   corresponding Lipschitz bound on $(0,\lambda_n)$.

Using the bounds \eqref{eq:G_cf_unif_bound}-\eqref{eq:G_cf_lambda_deriv_bound} we obtain the following covering number bound for all $\epsilon_1,\epsilon_2,\epsilon_3>0$:
\begin{align}
&N(\epsilon_1+\epsilon_2+\epsilon_3,\mathcal{G}_{c,f},\|\cdot\|_\infty)\leq \left\lceil\frac{\lambda_n C_1}{2\epsilon_1}\right\rceil\left\lceil\frac{(-s_0-\tilde{\nu}) \lambda_n C_2}{\epsilon_2}\right\rceil N(\epsilon_3/C_2,\mathcal{G},\|\cdot\|_\infty)\,,\\
&C_1\coloneqq f^*(-\tilde{\nu})-\inf f^*+\sup_{t\geq \tilde{\nu}}|t(f^*)^\prime_+(-t)|+(f^*)^\prime_+(-\tilde\nu)(\max\{-s_0,-\tilde{\nu}\}+\sup\tilde{c}  )\,,\notag\\
&C_2\coloneqq (f^*)^\prime_+(-\tilde{\nu})\,.\
\notag\end{align}
Combining this with the uniform bound \eqref{eq:G_cf_unif_bound}, we obtain a bound on the Rademacher complexity via the Dudley entropy integral:
\begin{align}
&2\mathcal{R}_{\mathcal{G}_{c,f},P,n}\notag\\
\leq& 24n^{-1/2}\int_0^{\beta C_2}\sqrt{\log\left(\left\lceil\frac{\lambda_n C_1}{2h_1(\tilde\epsilon)}\right\rceil\left\lceil\frac{(-s_0-\tilde{\nu}) \lambda_n C_2}{h_2(\tilde\epsilon)}\right\rceil N(h_3(\tilde\epsilon)/C_2,\mathcal{G},\|\cdot\|_\infty)\right)}d\tilde{\epsilon}\coloneqq \widetilde{R}_n\,.\label{eq:Rn_tilde_def}
\end{align}
 
Together, the results in Sections \ref{sec:OT_f_div_lambda_infinity}-\ref{sec:OT_reg_error_decomp} complete the   proof of Theorem \ref{thm:OT_reg_DRO_emp_objective_bound}.

\subsection{DRO with OT-Regularized $f$-Divergences: ERM Bound}
Similarly to Theorem \ref{thm:OT_DRO_emp_optimizer}, one can also show that a solution to the empirical OT-regularized $f$-divergence DRO problem (see \eqref{eq:OT_Reg_DRO_erm_sol} below) is an approximate solution to the population DRO problem with high probability (again allowing for an optimization error tolerance, $\epsilon_n^{\text{opt}}$, and optimization failure probability, $\delta_n^{\text{opt}}$).
\begin{theorem}\label{thm:OT_reg_DRO_conc}
Under Assumption \ref{asump:OT_reg_DRO}, let  $p_0\in(0,\min_y p_y)$ and $\tilde{\nu}\in\mathbb{R}$ such that 
\begin{align}
(f^*)^\prime_+(-M-\tilde{\nu})\geq 1/p_0\,.
\end{align}  Suppose we have $r>0$, $\theta_{*,n}:\mathcal{Z}^n\to \Theta$, $\epsilon_n^{\text{opt}}\geq 0$, ${\delta}^{\text{opt}}_n\in[0,1]$, and $E_n\subset \mathcal{Z}^n$ such that ${P}^n(E^c_n)\leq  {\delta}^{\text{opt}}_n$  and
\begin{align}\label{eq:OT_Reg_DRO_erm_sol}
\sup_{Q: D_f^c(Q\|P_{n})\leq r}E_Q[\mathcal{L}_{\theta_{*,n}}]\leq \inf_{\theta\in\Theta}\sup_{Q: D_f^c(Q\|P_{n})\leq r}E_Q[\mathcal{L}_\theta]+\epsilon_n^{\text{opt}}
\end{align}
on $E_n$, where $P_{n}\coloneqq\frac{1}{n}\sum_{i=1}^n \delta_{z_i}$.

 Then for $n\in\mathbb{Z}^+$, $\epsilon>0$   we have
\begin{align}\label{eq:OT_reg_conc_ineq_empirical_sol}
&{P}^n\left( \sup_{Q: D_f^c(Q\|P)\leq r}E_Q[\mathcal{L}_{\theta_{*,n}}]\geq \inf_{\theta\in\Theta}\sup_{Q: D_f^c(Q\|P)\leq r}E_Q[\mathcal{L}_\theta]+2\max\{R_n,\widetilde{R}_n\}+\epsilon^{\text{opt}}_n+\epsilon\right)\\
\leq&\delta_n^{opt}+2\exp\left(-\frac{n\epsilon^2}{ 2\beta^2}\right)+
2\exp\left(-\frac{n\epsilon^2}{2 (\beta(f^*)^\prime_+(-\tilde{\nu}))^2}\right)+2\sum_{y\in\mathcal{Y}} e^{-2n(p_y-p_0)^2}\,,\notag
\end{align}
where $R_n$ was defined in \eqref{R_n_def} and $\widetilde{R}_n$ was defined in \eqref{eq:Rn_tilde_def}.
\end{theorem}
The proof builds off of the proof of Theorem \ref{thm:OT_reg_DRO_emp_objective_bound}, similarly to how  Theorem \ref{thm:OT_DRO_emp_optimizer} built off of Theorem \ref{thm:conc_ineq_min_value}; for  details, see Appendix \ref{sm:OT_ref_ERM}.

\subsection{OT-Regularized $f$-Divergence Examples}\label{sec:examples_OT_reg_Df_DRO}
Here we show that the     KL divergence and  $\alpha$-divergence cases are covered by our results on OT-regularized $f$-divergence DRO; see the conditions in Assumption \ref{asump:OT_reg_DRO}.
\begin{enumerate}
    \item KL Divergence: In this case we have $f_{KL}(t)=t\log(t)$, $f_{KL}^*(t)=e^{t-1}$.  It is straightforward to verify that  $f_{KL}$ is strictly convex, $s_0=1\in \{f^*<\infty\}$, $\inf f_{KL}^*=0$, $\lim_{s\to \infty} (f_{KL}^*)^\prime(s)=\infty$, and
    \begin{align}
        \sup_{s\in[\tilde{\nu},\infty)}|s(f^*_{KL})^\prime(-s)|\leq \max\{e^{-2},-\tilde{\nu}e^{-\tilde{\nu}-1}\}\,.
    \end{align}
    Thus we see that $f_{KL}$ satisfies the requirements of Assumption \ref{asump:OT_reg_DRO} and hence Theorems \ref{thm:OT_reg_DRO_emp_objective_bound} and \ref{thm:OT_reg_DRO_conc} hold for any  
    $p_0\in(0,\min_y p_y)$ and any $\tilde{\nu}\leq -1-M-\log(1/p_0)$.

\item $\alpha$-Divergence, $\alpha>1$: Here, $f_\alpha$ is given by \eqref{eq:f_alpha_def} and $f_\alpha^*$ by \eqref{eq:f_alpha_star}. It is straightforward to verify that $f_\alpha$ is strictly convex, $s_0=\frac{1}{\alpha-1}\in\{f_\alpha^*<\infty\}=\mathbb{R}$, $\inf f_\alpha^*=\frac{1}{\alpha(\alpha-1)}$, $\lim_{s\to\infty}(f_\alpha^*)^\prime(s)=\infty$, and
\begin{align}
     \sup_{s\in[\tilde{\nu},\infty)}|s(f_{\alpha}^*)^\prime(-s)|=(\alpha-1)^{1/(\alpha-1)}\max\{-\tilde{\nu},0\}^{\alpha/(\alpha-1)}\,.
\end{align}
    Thus we see that for all $\alpha>1$, $f_{\alpha}$ satisfies the requirements of Assumption \ref{asump:OT_reg_DRO} and hence Theorems \ref{thm:OT_reg_DRO_emp_objective_bound} and \ref{thm:OT_reg_DRO_conc} hold for all $\tilde{\nu}\leq -M -(\alpha-1)^{-1} p_0^{-(\alpha-1)}$ and all 
    $p_0\in(0,\min_y p_y)$.
    \end{enumerate}

    Similarly to the examples in Section \ref{sec:examples_OT_DRO}, in both the KL and $\alpha$-divergence cases, if we consider $\psi$ satisfying $\psi^*(t)=O(t^{q/(q-1)})$ as $t\to 0$ for some $q>1$ (note that  the cases \eqref{eq:psi_case2} and \eqref{eq:psi_case3} satisfy this for all $q>1$) and let $\lambda_n=Cn^{r}$ with $r=\max\{(q-1)/2,1/2\}$ then $R_n=O(n^{-1/2})$ and $\tilde{R}_n=O(\sqrt{\log(n)/n})$ (provided the entropy integrals are finite, e.g., in the cases considered in Section \ref{sm:example_detials_DN}).

\section{Conclusions}
\label{sec:conclusions}
We derived statistical performance guarantees, in the form of concentration inequalities, for DRO with OT and OT-regularized $f$-divergence neighborhoods.  In the OT case, our results apply to both classification and regression and they improve on prior studies of OT-DRO in that they apply to a wider range of OT cost functions, beyond the $p$-Wasserstein case, and have improved dependence on the neighborhood size parameter. The increased generality of our results facilitates applications to a broader range of adversarial training methods.  We also provide the first study of statistical guarantees for OT-regularized $f$-divergence DRO, a recent class of methods that combine sample reweighting with adversarial sample generation. A drawback of our approach to  OT-regularized $f$-divergence DRO  is that our results only apply to a discrete label space (i.e., classification problems). Addressing this limitation is a direction we intend to explore in future work.

\appendix

 \section{Example Details}\label{sm:example_detials_DN}
 In this appendix we provide  bounds on the Dudley entropy integral term 
 \begin{align}\label{eq:D_n_def}
 D_n   \coloneqq& 12n^{-1/2}\int_0^{\beta}\sqrt{\log \left(\left(\left\lceil \frac{M\lambda_*(h_2(\tilde{\epsilon}))}{2h_2(\tilde{\epsilon})}\right\rceil+1\right) N(h_1(\tilde{\epsilon}),\mathcal{G},\|\cdot\|_\infty)\right)}d\tilde{\epsilon}\,,
\end{align}
  corresponding to each of the examples in Section \ref{sec:examples_OT_DRO}.
We consider the case where $\Theta$ is the unit ball in $\mathbb{R}^k$ with respect to some norm $\|\cdot\|_\Theta$ and we assume we have $L_\Theta\in(0,\infty)$ such that  $\theta\mapsto \mathcal{L}_\theta(z)$ is $L_\Theta$-Lipschitz under this norm for all $z\in\mathcal{Z}$.  In such cases we have the covering number bounds
\begin{align}\label{eq:N_bound_ball_Rk}
    N(\epsilon_1,\mathcal{G},\|\cdot\|_\infty)\leq N(\epsilon_1/L_\Theta,\Theta,\|\cdot\|_\Theta)\leq  (1+2L_\Theta/\epsilon_1)^k\,,
\end{align}
where we used, e.g., the result from Example 5.8 in \cite{wainwright2019high}.  Therefore, letting $h_1(\tilde{\epsilon})=\gamma\tilde{\epsilon}$, $h_2(\tilde{\epsilon})=(1-\gamma)\tilde{\epsilon}$ for $\gamma\in(0,1)$,  we obtain the following bounds on $D_n$.
\begin{enumerate}
    \item  $\psi(t)=\alpha t^q$ for some $\alpha>0$, $q>1$: Using \eqref{eq:lambda_star_case1} and \eqref{eq:N_bound_ball_Rk} we obtain
    \begin{align}
        D_n\leq& 12n^{-1/2}\int_0^{\beta}\sqrt{\log \left(\left(\left\lceil \frac{M(L_{\mathcal{X}}/q)^q(q-1)^{q-1}}{2\alpha((1-\gamma)\tilde{\epsilon})^{q}}\right\rceil+1\right)  \left(1+\frac{2L_\Theta}{\gamma\tilde{\epsilon}}\right)^k\right)}d\tilde{\epsilon}\\
        \leq&12n^{-1/2}\beta\int_0^{1}\sqrt{\log \left(\left\lceil \frac{M(L_{\mathcal{X}}/q)^q(q-1)^{q-1}}{2\alpha\beta^q((1-\gamma)u)^{q}}\right\rceil+1\right) +\frac{2kL_\Theta}{\beta \gamma u}}du\notag\\
         \leq&12n^{-1/2}\beta\int_0^{1}\sqrt{\log \left(\left( \frac{M(L_{\mathcal{X}}/q)^q(q-1)^{q-1}}{2\alpha\beta^q}+1\right)((1-\gamma)u)^{-q}+1\right) +\frac{2kL_\Theta}{\beta \gamma u}}du\notag\\
          \leq&24n^{-1/2}\beta\sqrt{q\left( \frac{M(L_{\mathcal{X}}/q)^q(q-1)^{q-1}}{2\alpha\beta^q}+1\right)(1-\gamma)^{-1} +\frac{2kL_\Theta}{\beta}\gamma^{-1}}\notag\,,
    \end{align}
    where we changed variables to $u=\tilde{\epsilon}/\beta$ and used that $\log(1+A)\leq A$, $\lceil A/s\rceil\leq(A+1)/s$ for all $s\in(0,1)$, $A\geq 0$, and $(1+Bt^{-q})\leq (1+B/t)^q$ for all $B\geq 1$, $t>0$.  Minimizing over $\gamma\in(0,1)$ and using the result
    \begin{align}
        \min_{\gamma\in(0,1)}\{ A(1-\gamma)^{-1}+B \gamma^{-1}\}
=\left(\sqrt{A}+\sqrt{B}\right)^2\,\,\,\text{ for all }A,B>0
    \end{align}
    we obtain
    \begin{align}
        &D_n\leq  24n^{-1/2}\beta\left(q^{1/2}\left( \frac{M(L_{\mathcal{X}}/q)^q(q-1)^{q-1}}{2\alpha\beta^q}+1\right)^{1/2}+\left(\frac{2kL_\Theta}{\beta}\right)^{1/2}\right)\,.
        \end{align}
\item $\psi(t)=\alpha t^q+\eta t$ for some $\alpha,\eta>0$, $q>1$: Using the bound $\lambda^*(\epsilon_2)\leq L_{\mathcal{X}}/\eta$ along with \eqref{eq:N_bound_ball_Rk} we can compute
\begin{align}
 D_n  \leq & 12n^{-1/2}\int_0^{\beta}\sqrt{\log \left(\left(\left\lceil \frac{ML_{\mathcal{X}}}{2\eta (1-\gamma)\tilde{\epsilon}}\right\rceil+1\right) \left(1+\frac{2L_\Theta}{\gamma\tilde{\epsilon}}\right)^k\right)}d\tilde{\epsilon}\\
 \leq& 12n^{-1/2}\int_0^{\beta}\sqrt{\left\lceil \frac{ML_{\mathcal{X}}}{2\eta (1-\gamma)\tilde{\epsilon}}\right\rceil +\frac{2kL_\Theta}{\gamma\tilde{\epsilon}}}d\tilde{\epsilon} \notag\\
 \leq&24n^{-1/2}\beta \sqrt{1+ \frac{ML_{\mathcal{X}}}{2\eta (1-\gamma)\beta} +\frac{2kL_\Theta}{\gamma \beta}} \,.\notag
\end{align}
Minimizing over $\gamma\in(0,1)$ we obtain
\begin{align}\label{eq:Dn_bound_case2}
 D_n  \leq &  24n^{-1/2}\beta \sqrt{1+ \frac{ML_{\mathcal{X}}}{2\eta \beta}+\frac{2kL_\Theta}{\beta}+\frac{2}{\beta}\sqrt{\frac{kML_{\mathcal{X}}L_\Theta }{\eta}}} \,.
\end{align}
\item $\psi(t)=\alpha (e^{qt}-1)$ for some $\alpha,q>0$: Mirroring the computation in the previous case, we obtain the bound \eqref{eq:Dn_bound_case2} except with $\eta$ replaced by $\alpha q$.
\item 
 $\psi(t)=\infty 1_{t>0}$: Using the fact that $\lambda_*=0$ in this case, we have
\begin{align} 
 D_n   =& 12n^{-1/2}\int_0^{\beta}\sqrt{\log \left( N(h_1(\tilde{\epsilon}),\mathcal{G},\|\cdot\|_\infty)\right)}d\tilde{\epsilon}\\
 \leq& 12(k/n)^{1/2}\int_0^{\beta}\sqrt{\log \left( 
 1+2L_\Theta/h_1(\tilde{\epsilon})
 \right)}d\tilde{\epsilon}\notag\\
 \leq&12(2L_\Theta k/n)^{1/2}\int_0^{\beta} h_1(\tilde{\epsilon})
^{-1/2}d\tilde{\epsilon}\notag\,.
\end{align}
Letting $h_1(\tilde{\epsilon})=\gamma\tilde{\epsilon}$ and taking $\gamma\to 1^-$ we find
\begin{align} 
 D_n    \leq& 24(2L_\Theta \beta k/n)^{1/2} \,.
\end{align}
\end{enumerate}

\section{Properties for the Functions $g_{\theta,\lambda,\nu}$}\label{sm:Lipschitz}
In this appendix we provide detailed derivations of the properties of the functions
\begin{align}
g_{\theta,\lambda,\nu}\coloneqq \lambda(f^*(-\nu)-f^*(-\Delta \mathcal{L}^c_{\theta,\lambda}/\lambda-\nu))\,,\,\,\, \theta\in\Theta,\lambda\in(0,\lambda_n),\nu\in[\tilde{\nu},-s_0]\,,
\end{align}
 that were stated and used in Section \ref{sec:OT_f_div_lambda_infinity} as part of the proof of Theorem \ref{thm:OT_reg_DRO_emp_objective_bound}; we work under Assumption \ref{asump:OT_reg_DRO}.
 
 First note that we have the following uniform bound on the   functions  in $\mathcal{G}_{c,f}$, which follows from the bounds   $0\leq \Delta\mathcal{L}^c_{\theta,\lambda}\leq \beta$, $-\Delta \mathcal{L}^c_{\theta,\lambda}/\lambda-\nu\leq -\nu\leq -\tilde{\nu}$ together with $f^*$ being non-decreasing and  $(f^*)^\prime_+(-\tilde{\nu})$-Lipschitz on $(-\infty,-\tilde{\nu}]$:
\begin{align} 
  0\leq \lambda(f^*(-\nu)-f^*(-\Delta \mathcal{L}^c_{\theta,\lambda}/\lambda-\nu)) \leq  \beta (f^*)^\prime_+(-\tilde{\nu}) 
\end{align}
for all $\theta\in\Theta,\lambda\in(0,\lambda_n),\nu\in[\tilde{\nu},-s_0]$.  Using these properties of $f^*$ we also obtain Lipschitz bounds in $\theta$,
\begin{align}
 &\left|\lambda(f^*(-\nu)-f^*(-\Delta \mathcal{L}^c_{\theta_1,\lambda}/\lambda-\nu))-   \lambda(f^*(-\nu)-f^*(-\Delta \mathcal{L}^c_{\theta_2,\lambda}/\lambda-\nu))\right|\\
  \leq&(f^*)^\prime_+(-\tilde{\nu})\left|  \mathcal{L}^c_{\theta_1,\lambda}- \mathcal{L}^c_{\theta_2,\lambda}\right|\notag\\
    \leq&(f^*)^\prime_+(-\tilde{\nu})\|  \mathcal{L}_{\theta_1}- \mathcal{L}_{\theta_2}\|_\infty\notag\,,
\end{align}
and in $\nu$,
\begin{align}
  &\left|   \lambda(f^*(-\nu_1)-f^*(-\Delta \mathcal{L}^c_{\theta,\lambda}/\lambda-\nu_1)) - \lambda(f^*(-\nu_2)-f^*(-\Delta \mathcal{L}^c_{\theta,\lambda}/\lambda-\nu_2)) \right|\\
  \leq& 2\lambda_n (f^*)^\prime_+(-\tilde{\nu})|\nu_1-\nu_2|\,.\notag
\end{align}

Lastly, we consider the dependence on $\lambda$. Start by noting that $f^*$ is Lipschitz on $(-\infty,d)$ for all $d\in\mathbb{R}$ and $-\Delta \mathcal{L}^c_{\theta,\lambda}/\lambda-\nu$ is bounded above and is absolutely continuous when $\lambda$ is restricted to a compact interval due to the convexity of $\mathcal{L}^c_{\theta,\lambda}$ in $\lambda$.  Therefore $\lambda\mapsto\lambda(f^*(-\nu)-f^*(-\Delta \mathcal{L}^c_{\theta,\lambda}/\lambda-\nu))$ is absolutely continuous when restricted to compact intervals and so it is differentiable a.s. We bound the derivative as follows. First compute
\begin{align}
&\partial_\lambda \left( \lambda(f^*(-\nu)-f^*(-\Delta \mathcal{L}^c_{\theta,\lambda}/\lambda-\nu))\right)\\
=&f^*(-\nu)-f^*(-\Delta \mathcal{L}^c_{\theta,\lambda}/\lambda-\nu)
- (f^*)^\prime_+(-\Delta \mathcal{L}^c_{\theta,\lambda}/\lambda-\nu)(-\partial_\lambda\Delta \mathcal{L}^c_{\theta,\lambda}  +\Delta \mathcal{L}^c_{\theta,\lambda}/\lambda)\,.\notag
\end{align}
By assumption, $f^*$ is bounded below and we also have that $f^*$ is non-decreasing, hence
\begin{align}
   0\leq f^*(-\nu)-f^*(-\Delta \mathcal{L}^c_{\theta,\lambda}/\lambda-\nu)\leq f^*(-\tilde{\nu})-\inf f^*\,. 
\end{align}
Using the bound
\begin{align}
  \left|\frac{\mathcal{L}^c_{\theta,\lambda+\Delta\lambda}-\mathcal{L}^c_{\theta,\lambda}}{\Delta \lambda}\right|  \leq \sup\tilde{c}\,,
\end{align}
together with  the assumption that $s(f^*)^\prime_+(-s)$ is bounded on $s\in[\tilde{\nu},\infty)$ for all $\tilde{\nu}\in\mathbb{R}$, we can then compute
\begin{align}
    &| (f^*)^\prime_+(-\Delta \mathcal{L}^c_{\theta,\lambda}/\lambda-\nu)(-\partial_\lambda\Delta \mathcal{L}^c_{\theta,\lambda}  +\Delta \mathcal{L}^c_{\theta,\lambda}/\lambda)|\\
    \leq&| (f^*)^\prime_+(-\Delta \mathcal{L}^c_{\theta,\lambda}/\lambda-\nu)(\Delta \mathcal{L}^c_{\theta,\lambda}/\lambda+\nu)|+|(f^*)^\prime_+(-\Delta \mathcal{L}^c_{\theta,\lambda}/\lambda-\nu)(\nu+\partial_\lambda\Delta \mathcal{L}^c_{\theta,\lambda})  )|\notag\\
        \leq&\sup_{t\geq \tilde{\nu}}|t(f^*)^\prime_+(-t)|+(f^*)^\prime_+(-\tilde\nu)(\max\{-s_0,-\tilde{\nu}\}+\sup\tilde{c}   )\,.\notag
\end{align}
Therefore we obtain
\begin{align}
&\left|\partial_\lambda \left( \lambda(f^*(-\nu)-f^*(-\Delta \mathcal{L}^c_{\theta,\lambda}/\lambda-\nu))\right)\right|\\
\leq &f^*(-\tilde{\nu})-\inf f^*+\sup_{t\geq \tilde{\nu}}|t(f^*)^\prime_+(-t)|+(f^*)^\prime_+(-\tilde\nu)(\max\{-s_0,-\tilde{\nu}\}+\sup\tilde{c}  )\,.\notag
\end{align}
This completes the proof of the properties of the functions $g_{\theta,\lambda,\nu}$ which were used in the main text.

\section{Proof of ERM Bound for DRO with OT-Regularized $f$-Divergences}\label{sm:OT_ref_ERM}

In this appendix we prove Theorem \ref{thm:OT_reg_DRO_conc} from the main text, which is repeated below.
\begin{theorem}\label{thm:OT_reg_DRO_conc_sm}
Under Assumption \ref{asump:OT_reg_DRO}, let  $p_0\in(0,\min_y p_y)$ and $\tilde{\nu}\in\mathbb{R}$ such that 
\begin{align}
(f^*)^\prime_+(-M-\tilde{\nu})\geq 1/p_0\,.
\end{align}  Suppose we have $r>0$, $\theta_{*,n}:\mathcal{Z}^n\to \Theta$, $\epsilon_n^{\text{opt}}\geq 0$, ${\delta}^{\text{opt}}_n\in[0,1]$, and $E_n\subset \mathcal{Z}^n$ such that ${P}^n(E^c_n)\leq  {\delta}^{\text{opt}}_n$  and
\begin{align}\label{eq:OT_Reg_DRO_erm_sol_sm}
\sup_{Q: D_f^c(Q\|P_{n})\leq r}E_Q[\mathcal{L}_{\theta_{*,n}}]\leq \inf_{\theta\in\Theta}\sup_{Q: D_f^c(Q\|P_{n})\leq r}E_Q[\mathcal{L}_\theta]+\epsilon_n^{\text{opt}}
\end{align}
on $E_n$, where $P_{n}\coloneqq\frac{1}{n}\sum_{i=1}^n \delta_{z_i}$.

 Then for $n\in\mathbb{Z}^+$, $\epsilon>0$   we have
\begin{align}\label{eq:OT_reg_conc_ineq_empirical_sol_sm}
&{P}^n\left( \sup_{Q: D_f^c(Q\|P)\leq r}E_Q[\mathcal{L}_{\theta_{*,n}}]\geq \inf_{\theta\in\Theta}\sup_{Q: D_f^c(Q\|P)\leq r}E_Q[\mathcal{L}_\theta]+2\max\{R_n,\widetilde{R}_n\}+\epsilon^{\text{opt}}_n+\epsilon\right)\\
\leq&\delta_n^{opt}+2\exp\left(-\frac{n\epsilon^2}{ 2\beta^2}\right)+
2\exp\left(-\frac{n\epsilon^2}{2 (\beta(f^*)^\prime_+(-\tilde{\nu}))^2}\right)+2\sum_{y\in\mathcal{Y}} e^{-2n(p_y-p_0)^2}\,,\notag
\end{align}
where $R_n$ was defined in \eqref{R_n_def} and $\widetilde{R}_n$ was defined in \eqref{eq:Rn_tilde_def}.
\end{theorem}
\begin{proof}
   Using \eqref{eq:OT_Reg_DRO_erm_sol_sm}, on $E_n$ one can compute the  error decomposition bound
   \begin{align}
     &\sup_{Q: D_f^c(Q\|P)\leq r}E_Q[\mathcal{L}_{\theta_{*,n}}]-    \inf_{\theta\in\Theta}\sup_{Q: D_f^c(Q\|P)\leq r}E_Q[\mathcal{L}_\theta]\\
     \leq&\sup_{Q: D_f^c(Q\|P)\leq r}E_Q[\mathcal{L}_{\theta_{*,n}}]-\sup_{Q: D_f^c(Q\|P_{n})\leq r}E_Q[\mathcal{L}_{\theta_{*,n}}]\notag\\
     &+ \inf_{\theta\in\Theta}\sup_{Q: D_f^c(Q\|P_{n})\leq r}E_Q[\mathcal{L}_\theta]-    \inf_{\theta\in\Theta}\sup_{Q: D_f^c(Q\|P)\leq r}E_Q[\mathcal{L}_\theta]+\epsilon_n^{\text{opt}}\notag\\
     \leq&\sup_{\theta\in\Theta}\left\{\sup_{Q: D_f^c(Q\|P)\leq r}E_Q[\mathcal{L}_{\theta}]-\sup_{Q: D_f^c(Q\|P_n)\leq r}E_Q[\mathcal{L}_{\theta}]\right\}\notag\\
     &+\sup_{\theta\in\Theta}\left\{-\left(\sup_{Q: D_f^c(Q\|P)\leq r}E_Q[\mathcal{L}_{\theta}]-\sup_{Q: D_f^c(Q\|P_n)\leq r}E_Q[\mathcal{L}_{\theta}]\right)\right\}+\epsilon_n^{\text{opt}}\notag\,.
   \end{align}
   The claimed result then follows from   a union bound followed by essentially the  same computations as in the proof of Theorem \ref{thm:OT_reg_DRO_emp_objective_bound}.
\end{proof}

\bibliography{ConcIneqOTRegDRO_arxiv.bbl}

\end{document}